\documentclass[10pt,twocolumn,letterpaper]{article}

\usepackage{cvpr}              
\newcommand{\red}[1]{{\color{red}#1}}

\usepackage{multirow}
\usepackage{booktabs}
\usepackage{amsmath}
\usepackage{amsfonts}
\usepackage{xcolor}
\usepackage{colortbl}
\usepackage{algorithm,algpseudocode}
\usepackage{bbm}
\usepackage{pifont} 
\usepackage{tablefootnote}

\newcommand{\thickhat}[1]{\mathbf{\hat{\text{$#1$}}}}

\usepackage{amsthm}
\theoremstyle{plain}
\newtheorem{proposition}{Proposition}
\theoremstyle{remark}
\newtheorem{remark}{Remark}

\definecolor{blue}{RGB}{0, 10, 255}
\newcommand{\blue}[1]{{\color{blue}#1}}

\definecolor{cvprblue}{rgb}{0.21,0.49,0.74}
\usepackage[pagebackref,breaklinks,colorlinks,allcolors=cvprblue]{hyperref}

\def\paperID{4250} 
\def\confName{CVPR}
\def\confYear{2025}

\title{EfficientViM: Efficient Vision Mamba with \\Hidden State Mixer based State Space Duality}

\author{\textbf{Sanghyeok Lee}$^1$ \hspace{0.4cm} \textbf{Joonmyung Choi}$^1$ \hspace{0.4cm}  \textbf{Hyunwoo J. Kim}$^2$\thanks{Corresponding author.} \vspace{0.4cm}
\\
$^1$Korea University \hspace{0.4cm}
$^2$KAIST\\
{\tt\small \{\href{mailto:cat0626@korea.ac.kr}{cat0626}, \href{mailto:pizard@korea.ac.kr}{pizard}\}@korea.ac.kr \hspace{0.4cm}
\href{mailto:hyunwoojkim@korea.ac.kr}{hyunwoojkim}@kaist.ac.kr}
}
\begin{document}
\maketitle
\begin{abstract}
For the deployment of neural networks in resource-constrained environments, prior works have built lightweight architectures with convolution and attention for capturing local and global dependencies, respectively.
Recently, the state space model (SSM) has emerged as an effective operation for global interaction with its favorable linear computational cost in the number of tokens.
To harness the efficacy of SSM, we introduce Efficient Vision Mamba (\textbf{EfficientViM}), a novel architecture built on hidden state mixer-based state space duality (\textbf{HSM-SSD}) that efficiently captures global dependencies with further reduced computational cost.
With the observation that the runtime of the SSD layer is driven by the linear projections on the input sequences, we redesign the original SSD layer to perform the channel mixing operation within compressed hidden states in the HSM-SSD layer.
Additionally, we propose multi-stage hidden state fusion to reinforce the representation power of hidden states and provide the design to alleviate the bottleneck caused by the memory-bound operations.
As a result, the EfficientViM family achieves a new state-of-the-art speed-accuracy trade-off on ImageNet-1k, offering up to a 0.7\% performance improvement over the second-best model SHViT with faster speed.
Further, we observe significant improvements in throughput and accuracy compared to prior works, when scaling images or employing distillation training.
Code is available at \url{https://github.com/mlvlab/EfficientViM}.
\end{abstract}
\begin{figure}[t!]
     \centering
     \includegraphics[width=\columnwidth]{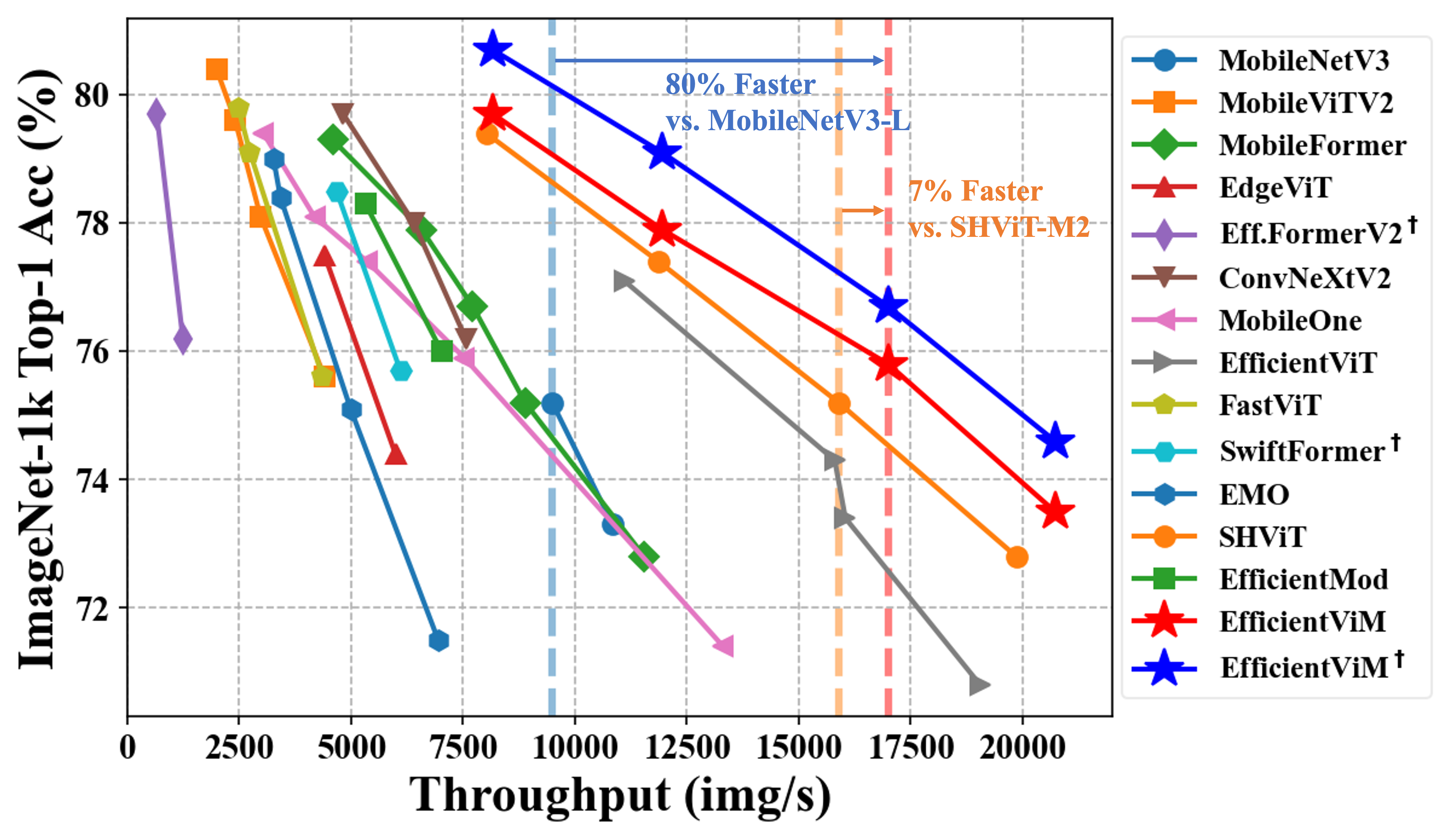}
     \vspace{-15pt}
    \caption{\textbf{Comparison of efficient networks on ImageNet-1K~\cite{deng2009imagenet} classification.} The family of our EfficientViM, marked as \red{red} and \blue{blue} stars, shows the best speed-accuracy trade-offs. \ding{61} indicates the model trained with distillation following~\cite{touvron2021training}.}
    \label{fig:intro}
    \vspace{-15pt}
\end{figure}
\section{Introduction}
Efficient vision architectures for resource-constrained environments have been consistently explored in computer vision tasks, including image classification~\cite{howard2017mobilenets,iandola2016squeezenet,tan2019efficientnet,zhang2018shufflenet,chollet2017xception,mehta2021mobilevit,mehta2022separable}, object detection~\cite{tan2020efficientdet,wang2024gold,li2017mimicking,redmon2016you}, segmentation~\cite{xie2021segformer,Xiong_2024_CVPR,xu2023pidnet,hu2023you,zhao2018icnet}, etc. 
Earlier works~\cite{chollet2017xception,han2020ghostnet,ma2018shufflenet} have explored the efficient convolutional neural networks (CNN).
One such technique is depthwise separable convolution (DWConv), introduced in Xception~\cite{chollet2017xception}, which has been widely adopted for modeling lightweight architectures, especially CNN-based networks including MobileNet~\cite{howard2017mobilenets,sandler2018mobilenetv2,howard2019searching} and following works~\cite{vasu2023mobileone,tan2019efficientnet}.\\
Meanwhile, with the advent of Vision Transformer (ViT) \cite{dosovitskiy2020image}, attention mechanisms have become a key operation for capturing long-range dependencies within image patches.
However, the quadratic cost of self-attention presents challenges for designing efficient architectures.
To address this, prior works have attempted to approximate self-attention with reduced cost~\cite{wang2020linformer,kitaev2020reformer,xiong2021nystromformer,choromanski2020rethinking}, or restrict the number of tokens~\cite{liu2021swin,chu2021twins, bolya2022token,lee2024multi,yin2022vit,choi2024vid}.
For on-device deployment, several works~\cite{mehta2021mobilevit, mehta2022separable,vasu2023fastvit, zhang2023rethinking, chen2022mobile, liu2023efficientvit,yun2024shvit,pan2022edgevits} devoted effort to developing hybrid ViTs combined with CNNs (DWConv).
While prior works have demonstrated superior performance over traditional CNNs, the inherent quadratic complexity of self-attention in the number of tokens remains a major bottleneck limiting their efficiency and scalability.
\\
Recently, state space models (SSMs)~\cite{fu2022hungry,gu2021efficiently,smith2022simplified,mamba,mamba2} have emerged as a promising alternative to self-attention, offering a favorable linear computational complexity while maintaining a global receptive field.
Mamba~\cite{mamba,mamba2} has introduced selective scanning mechanisms on SSM (S6) to process sequences with the hardware-aware parallel algorithm.
Following Mamba, vision Mambas~\cite{zhu2024vision, liu2024vmamba,yang2024plainmamba,huang2024localmamba} have extended the concept of SSM to vision tasks.
These works have studied multi-path scanning mechanisms to address the causal constraints of SSM that are not desirable for image processing.
More recent works like VSSD~\cite{shi2024vssd} and Linfusion~\cite{liu2024linfusion} further eliminated the causal mask in the state space duality (SSD) of Mamba2~\cite{mamba2}, introducing non-causal state space duality (NC-SSD).
While vision Mambas demonstrate improved performance over previous SOTA methods, they still have relatively slower speeds than lightweight vision models.
Further, we have observed that the major bottleneck of the previous SSD is the linear projection in gating operation and output projection.
\\
\noindent In this paper, we present an \textbf{Efficient} \textbf{Vi}sion \textbf{M}amba (\textbf{EfficientViM}), a new family of mamba-based lightweight vision backbone built with a fast and effective SSD layer called \textbf{H}idden \textbf{S}tate \textbf{M}ixer-based \textbf{SSD} (\textbf{HSM-SSD}).
In the HSM-SSD layer, we transfer the channel mixing operations of the standard SSD layer, including linear projection and gating function, from the image feature space to the hidden state space. 
These hidden states serve as compressed latent representations of the input.
We observe that this design mitigates the major bottleneck of the SSD layer while maintaining the generalization ability of the model.\\
Also, we introduce a multi-stage hidden state fusion approach that generates the predictions by combining the original logits with those derived from the hidden states at each stage, leading to an enhanced representation power of hidden states.
After breaking down the runtime of the HSM-SSD layer, we present the macro design that minimizes memory-bound operations, prioritizing practical performance in real-world applications over theoretical metrics like FLOPs.
Through extensive experiments, we demonstrate that our EfficientViM achieves the new speed-accuracy state-of-the-art trade-off as shown in~\Cref{fig:intro}.
In particular, EfficientViM-M2 outperforms the previous SOTA model SHViT~\cite{yun2024shvit}, and pioneering work ModelNetV3~\cite{howard2019searching} with a 0.6\% performance improvement even bringing about 7\% and 80\% speed-ups, respectively. In summary, the contributions of EfficientViM are threefold:
\begin{itemize}
    \item We propose a novel mamba-based lightweight architecture called EfficientViM, leveraging the linear computational cost of the global token mixer. 
    \item We introduce HSM-SSD, which makes the major overhead of the SSD layer controllable by adjusting the number of hidden states.
    \item With a design minimizing memory-bound operations and incorporating multi-stage hidden state fusion, EfficientViM achieves the best speed-accuracy tradeoffs.
\end{itemize}

\section{Preliminaries}
\paragraph{State space models (SSM).} 
Inspired by the linear time-invariant (LTI) continuous system, an SSM maps an input sequence $x(t) \in \mathbb{R}$ to an output sequence $y(t) \in \mathbb{R}$ as
\begin{equation}
h^\prime(t) = \thickhat{\mathbf{A}}h(t)+ \thickhat{\mathbf{B}}x(t),\;\; y(t) = \thickhat{\mathbf{C}} h(t),
\end{equation}
where $h(t) \in \mathbb{R}^{N \times 1}$ is a hidden state, $\thickhat{\mathbf{A}} \in \mathbb{R}^{N\times N}$, $\thickhat{\mathbf{B}} \in \mathbb{R}^{N\times1}$, $\hat{\mathbf{C}} \in \mathbb{R}^{1\times N}$  are the projection matrix, and $N$ is the number of states.
To adapt this continuous-time system for discrete data in deep learning, given the multivariate input sequences $\mathbf{x} = \left[\mathbf{x}^\top_1, \ldots,\mathbf{x}^\top_L\right]^\top \in \mathbb{R}^{L \times D}$ with $\forall_t \mathbf{x}_t  \in \mathbb{R}^{1\times D}$, Mamba~\cite{mamba} first generates parameters as 
$\thickhat{\mathbf{B}}, \mathbf{C}=\mathbf{x}\mathbf{W}_\mathbf{B}, \mathbf{x}\mathbf{W}_\mathbf{C} \in \mathbb{R}^{L \times N}$, $\Delta = \mathbf{x}\mathbf{W}_\Delta \in \mathbb{R}^L$, where $\mathbf{W}_\mathbf{B}, \mathbf{W}_\mathbf{C} \in \mathbb{R}^{D\times N}$, $\mathbf{W}_\Delta \in \mathbb{R}^{D \times 1}$ are learnable matrices.
Then, the discretized form of SSM with zero-order hold discretization is defined as 
\begin{equation}
\mathbf{h}_{t} = \mathbf{A}_t\mathbf{h}_{t-1} +  \mathbf{B}^\top_t \mathbf{x}_{t},
\;\; \mathbf{y}_{t} = \mathbf{C}_t\mathbf{h}_{t},
\label{eq:SSM}
\end{equation} 
where $\mathbf{y} \in \mathbb{R}^{L \times D}$, $\mathbf{h}_t \in \mathbb{R}^{N \times D}$, $\mathbf{A}_t = e^{\Delta_t \thickhat{\mathbf{A}}} \in \mathbb{R}^{N\times N}$,
$\mathbf{B}_t = (\Delta_t \thickhat{\mathbf{A}})^{-1}(e^{\Delta_t \thickhat{\mathbf{A}}} - \mathbf{I})\cdot\Delta_t \thickhat{\mathbf{B}}_t \approx \Delta_t \thickhat{\mathbf{B}}_t \in \mathbb{R}^{1 \times N}$.
In this formulation, $\hat{\mathbf{A}} \in \mathbb{R}^{N\times N}$ is a learnable diagonal matrix, and all projection matrices $\mathbf{A}_t,\mathbf{B}_t, \mathbf{C}_t$ enable the linear time-variant discrete system that selectively attends to inputs $\mathbf{x}$ and hidden state $\mathbf{h}$ of each timestamp $t$.

\paragraph{State space duality (SSD).} 
Mamba2~\cite{mamba2} further simplifies the diagonal form of the evolution matrix $\hat{\mathbf{A}}$ into the scalar form as $\hat{a} \in \mathbb{R}$ resulting in $\mathbf{a} \in \mathbb{R}^{L}$ via the same discretization step.
Then, the state space duality (SSD) reformulates \Cref{eq:SSM} as matrix transformation:
\begin{equation}
\begin{aligned}
\mathbf{y} & = \text{SSD}(\mathbf{x}, \mathbf{a},\mathbf{B},\mathbf{C}) = \left(\mathbf{M} \odot (\mathbf{C}\mathbf{B}^\top) \right) \mathbf{x} ,\\
\mathbf{M}_{ij} &= 
\begin{cases}
\prod^i_{k=j+1} \mathbf{a}_k \; &\text{if}\; i>j\\
1  &\text{if}\; i=j\\
0  &\text{if}\; i<j
\end{cases},
\end{aligned}
\label{eq:SSD}
\end{equation}
where $\mathbf{M} \in \mathbb{R}^{L \times L}$, and  $\odot$ indicates Hadamard product.
Note that the lower triangular matrix $\mathbf{M}$ acts as a causal mask, which is suboptimal for image processing.
To address this, non-causal SSD (NC-SSD)~\cite{zhu2024vision, liu2024linfusion} has been studied as an alternative to SSD by defining the mask as $\mathbf{M}_{ij} = \prod^{L}_{k=j+1} \mathbf{a}_k$, resulting in $\mathbf{M}_{1j}=\mathbf{M}_{2j}=\ldots=\mathbf{M}_{Lj}$.
Further, in VSSD~\cite{zhu2024vision}, it is  simplified as $\mathbf{M}_{ij} = \mathbf{a}_j$ resulting in
\begin{equation}
\begin{aligned}
     \mathbf{y} &= \text{NC-SSD}(\mathbf{x}, \mathbf{a},\mathbf{B},\mathbf{C}) = \mathbf{C}\mathbf{h},\\
     \mathbf{h} &= (\mathbf{a}\mathbbm{1}_N^\top \odot \mathbf{B})^\top \mathbf{x} = {\sum}^{L}_{i=1} \mathbf{a}_i \mathbf{B}_i^\top \mathbf{x}_i,
\end{aligned}
\label{eq:NCSSD}
\end{equation}
where $\mathbbm{1}_N \in \mathbb{R}^{N}$ is one vector for broadcasting $\mathbf{a}$.
Since the cumulative multiplication of $\mathbf{a}$ restricts the receptive fields as discussed in~\cite{shi2024multi, shi2024vssd}, we adopt this version of NC-SSD as our starting point for an efficient token mixer.

\section{Method}
\label{sec:3}
We present a hidden state mixer-based SSD (\textbf{HSM-SSD}) for capturing global context with reduced costs, detailed in~\Cref{sec:3.1}.
Then, we discuss the additional techniques to improve both speed and performance with HSM-SSD in~\Cref{sec:3.2}. 
After that, we outline the overall architecture and block designs to construct \textbf{EfficientViM} in~\Cref{sec:3.3}.

\subsection{Hidden State Mixer-based SSD}
\label{sec:3.1}
\begin{table}[t!]
    \centering
    \setlength{\tabcolsep}{2pt}
    \begin{tabular}{cc}
    \toprule
    Complexity & \multicolumn{1}{c}{Method}  \\
    \midrule
    $\mathcal{O}(LD^2 + L^2D)$ & Attention~\cite{vaswani2017attention}    \\
    $\mathcal{O}(LD^2)$  & Linear Attention~\cite{katharopoulos2020transformers} \\
    $\mathcal{O}(LD^2 + LND)$  &  SSD~\cite{mamba2}, NC-SSD~\cite{shi2024vssd}    \\
    $\mathcal{O}(N D^2 + LND)$  & HSM-SSD  \\
    \bottomrule
    \end{tabular}
    \caption{\textbf{Complexity comparison of global tokens mixers}. $L$: \# tokens, $N$: \# states, $D$: \# channels.}
    \label{tab:complexity} 
    \vspace{-10pt}
\end{table}

We start with a brief discussion on the computational cost of the NC-SSD layer depicted in~\Cref{fig:NCSSD_layer}.
The entire process of the NC-SSD layer can be summarized as
\begin{align}
    &\thickhat{\mathbf{B}}, \mathbf{C}, \Delta, \mathbf{x}, \mathbf{z} = \text{Linear}(\mathbf{x}_\text{in}) \label{eq:linear}\\
    &\mathbf{a}, \mathbf{B} = \text{Discretization}(\thickhat{a}, \thickhat{\mathbf{B}}, \Delta),\label{eq:discret}\\
    &\mathbf{B}, \mathbf{C}, \mathbf{x} := \text{DWConv}(\mathbf{B}, \mathbf{C}, \mathbf{x}),\label{eq:DWconv}\\
    &\mathbf{y} = \text{NC-SSD}(\mathbf{\mathbf{x}}, \mathbf{a}, \mathbf{B}, \mathbf{C}),\\
    &\mathbf{x}_\text{out} = \text{Linear}(\mathbf{y} \odot \sigma(\mathbf{z})), \label{eq:gating}
\end{align}
where $\mathbf{x}_\text{in}, \mathbf{x}_\text{out}, \mathbf{x}, \mathbf{z} \in \mathbb{R}^{L \times D}$, and $\sigma$ is the activation function.
The computational costs of \Cref{eq:linear,eq:discret,eq:DWconv} with a constant kernel size is $\mathcal{O}(LD^2 + LND)$.
Subsequently, executing NC-SSD and the output projection requires $\mathcal{O}(LND)$ and $\mathcal{O}(LD^2)$, respectively.
Given that the number of states $N$ is typically much smaller than the number of channels $D$ (i.e., $N \ll D$), the overall complexity is mainly driven by the linear projections involved in generating $\mathbf{x}$, $\mathbf{z}$, and $\mathbf{x}_\text{out}$, leading to $\mathcal{O}(LD^2)$.
Therefore, optimizing the linear projection in SSD blocks is crucial for scalability.\\
\noindent 
We delve into optimizing these computations for efficient layer designs.
NC-SSD (\Cref{eq:NCSSD}) can be factorized into two steps.
First, it obtains shared global hidden state $\mathbf{h} \in \mathbb{R}^{N \times D}$ through a weighted linear combination of input states $\mathbf{B}^\top_i\mathbf{x}_i \in \mathbb{R}^{N \times D}$ using importance weights $\mathbf{a} \in \mathbb{R}^{L}$.
Second, the outputs for each input are generated by projecting a hidden state with their corresponding $\mathbf{C} \in \mathbb{R}^{L \times N}$.
Here, if we denote the projected input $\mathbf{x}$ as $\mathbf{x}_\text{in}\mathbf{W}_\text{in}$ removing DWConv, the following holds:
\begin{equation}
\begin{aligned}
     \mathbf{h} &= (\mathbf{a} \mathbbm{1}^\top_N\odot \mathbf{B})^\top (\mathbf{x}_\text{in}\mathbf{W}_\text{in}) \\
     &=((\mathbf{a} \mathbbm{1}^\top_N \odot \mathbf{B})^\top \mathbf{x}_\text{in})\mathbf{W}_\text{in} = \mathbf{h}_\text{in} \mathbf{W}_\text{in},
     \label{eq:hidden}
\end{aligned}
\end{equation}
where $\mathbf{W}_\text{in} \in \mathbb{R}^{D \times D}$, and $\mathbf{h}_\text{in} = (\mathbf{a} \mathbbm{1}^\top_N \odot \mathbf{B})^\top \mathbf{x}_\text{in}  \in \mathbb{R}^{N\times D}$.
By computing $\mathbf{h}_{\text{in}}$ first, we perform a linear projection onto the hidden state space.
This approach reduces the cost from $\mathcal{O}(L D^2)$ to $\mathcal{O}(N D^2)$ which relies on the number of states.
In other words, we could alleviate the major costs of the layer by adjusting the number of states such that $N \ll L$.

\vspace{5pt}
\noindent\textbf{Hidden State Mixer.}
The next step is to alleviate the cost of gating and output projection in~\Cref{eq:gating}, which still remains $\mathcal{O}({LD^2})$.
To address this, we focus on a shared global hidden state $\mathbf{h}$.
Note that the hidden state $\mathbf{h}$ itself is the reduced latent array that compresses the input data with a significantly smaller sequence length $N$. 
Based on this observation, we propose a \textit{Hidden State Mixer} (\textbf{HSM}) that performs channel mixing, including the gating and output projection, directly on the reduced latent array $\mathbf{h}$ as highlighted in~\Cref{fig:HSMSSD_layer}.
To this end, we approximate the output of the NC-SSD layer as follows:
\begin{figure}[t!]
     \centering
    \begin{subfigure}{0.58\columnwidth}
        \centering
        \includegraphics[height=6.7cm]{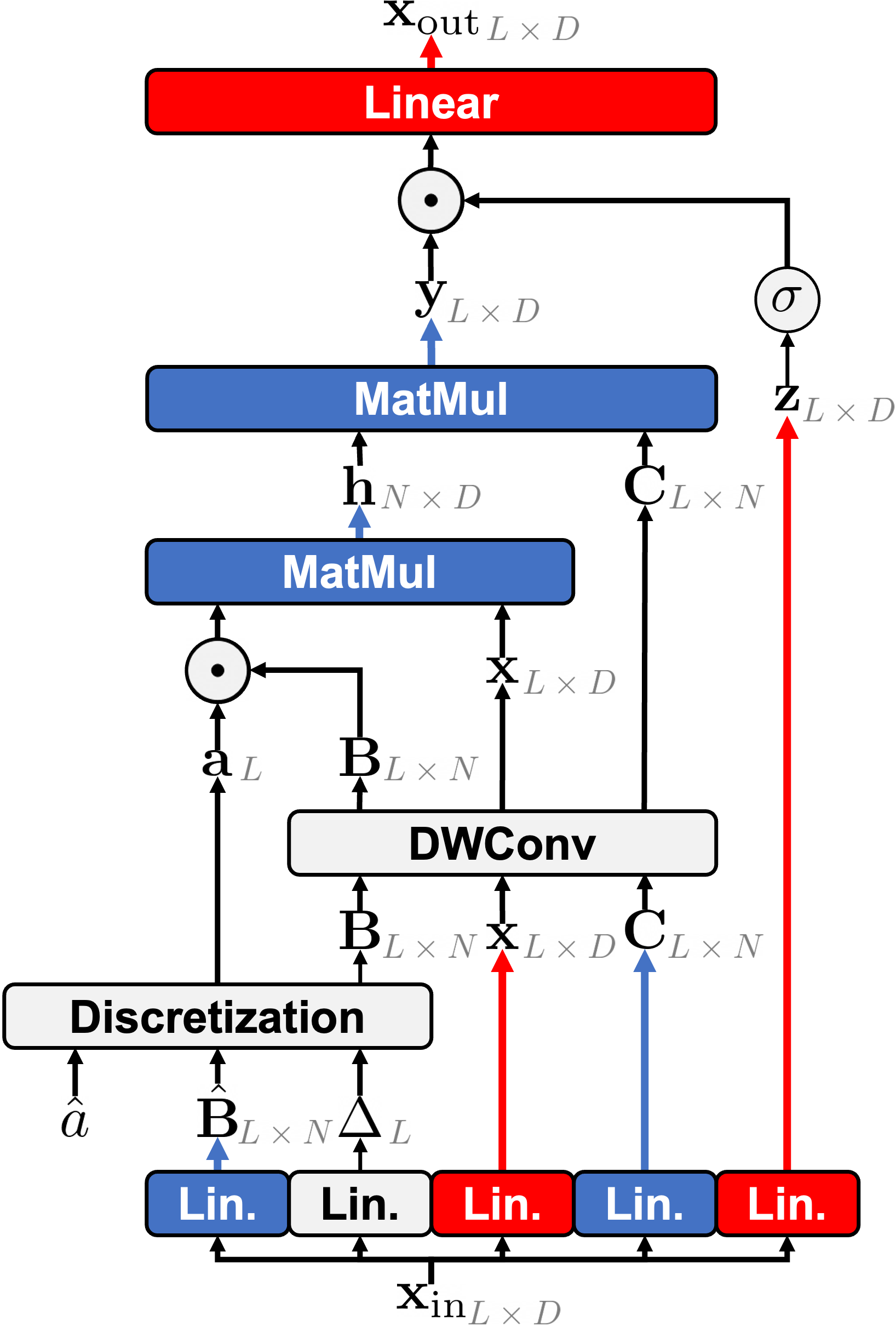}
        \caption{NC-SSD layer}
        \label{fig:NCSSD_layer}
    \end{subfigure}
    \hfill
    \begin{subfigure}{0.40\columnwidth}
        \centering
        \includegraphics[height=6.7cm]{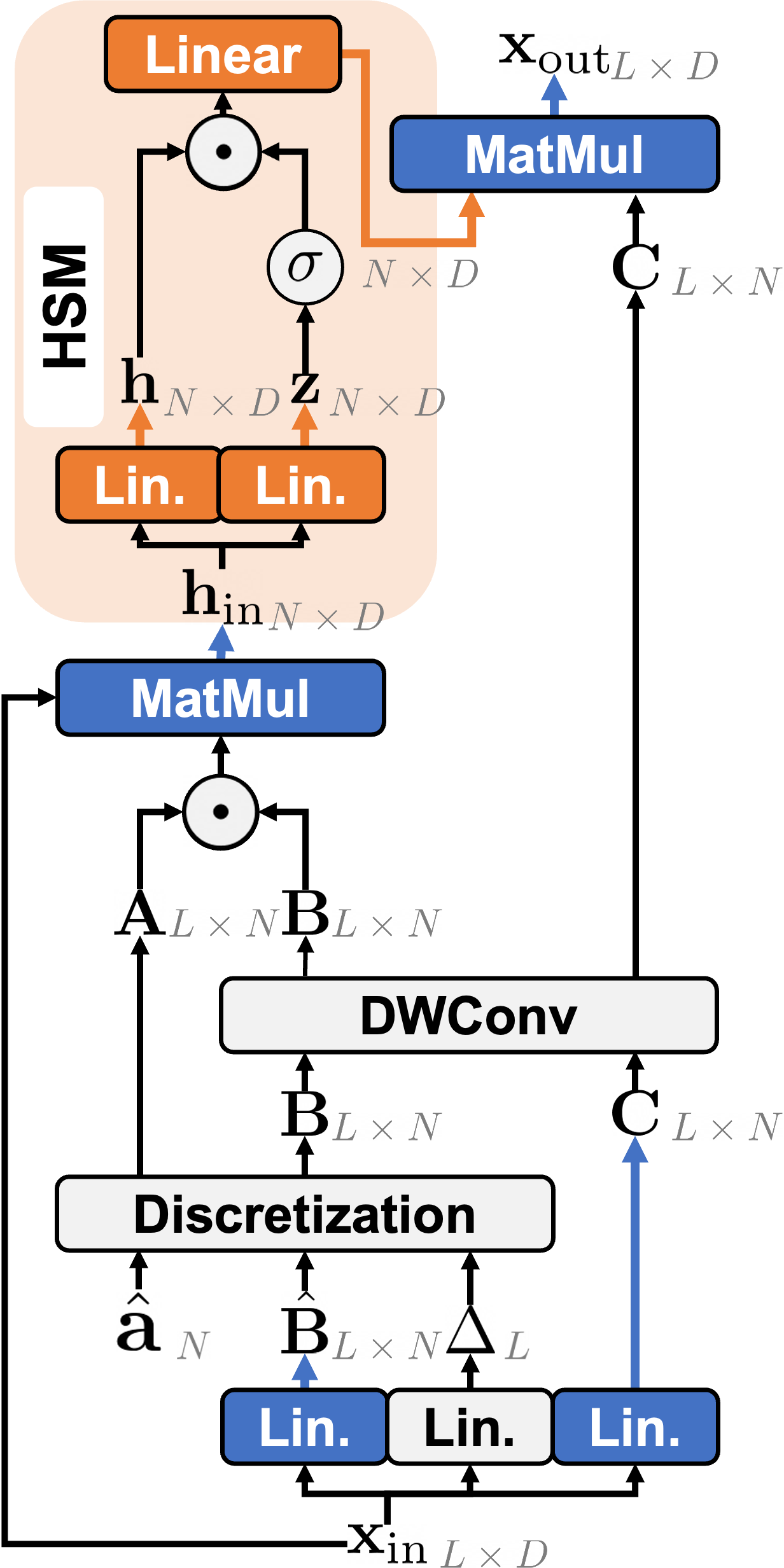}
        \caption{HSM-SSD layer}
        \label{fig:HSMSSD_layer}
    \end{subfigure}
    \vspace{-15pt}
    \definecolor{blue}{RGB}{69, 114, 196}
    \definecolor{orange}{RGB}{237, 125, 49}
    \caption{\textbf{Illustration of (left) NC-SSD and (right) HSM-SSD layer.} In the HSM-SSD layer, the computationally heavy projections are handled with the reduced hidden state in HSM as highlighted. \red{Red}, \textcolor{blue}{blue}, and \textcolor{orange}{orange} colors indicate the operation requiring the complexities of \red{$\mathcal{O}(LD^2)$}, \textcolor{blue}{$\mathcal{O}(LND)$}, and \textcolor{orange}{$\mathcal{O}(ND^2)$}.}
    \label{fig:layer}
    \vspace{-15pt}
\end{figure}
\begin{equation}
\begin{aligned}
    \mathbf{x}_\text{out} &= f(\mathbf{y})\\
    & = \text{Linear}(\mathbf{y} \odot \sigma(\mathbf{z}))\\
    &=(\mathbf{C}\mathbf{h}
    \odot \sigma{(\mathbf{x}_\text{in} \mathbf{W}_\mathbf{z}}))\mathbf{W}_\text{out}\\
    &\approx \mathbf{C}\left((\mathbf{h}
    \odot \sigma{(\mathbf{h}_\text{in} \mathbf{W}_\mathbf{z}}))\mathbf{W}_\text{out}\right) = \mathbf{C}f(\mathbf{h}),
\end{aligned}
\label{eq:HSM}
\end{equation}
\noindent where $\mathbf{y} =\mathbf{Ch}$ from~\Cref{eq:NCSSD}, and $f$ indicates channel mixing of gating function followed by linear projection with the learnable matrix $\mathbf{W}_\mathbf{z}, \mathbf{W}_\text{out} \in \mathbb{R}^{D \times D}$.
Contrary to the original NC-SSD layer where $\mathbf{C}\mathbf{h}$ is computed first and then fed to $f$, we apply the gating and projection directly to the hidden states with HSM.
Then, the final output $\mathbf{x}_\text{out}$ is generated by projecting updated hidden states with $\mathbf{C}$.
Consequently, the total complexity of capturing global context in the HSM-SSD layer becomes $\mathcal{O}({ND^2 + LND})$, which is negligible as $N$ gets smaller.
Refer to~\Cref{tab:complexity} for a comparison of the big-$\mathcal{O}$ complexities with previous global token mixers.
\begin{proposition}
Let $N=L$, $\mathbf{a} \mathbbm{1}^\top_L \odot \mathbf{B} = \mathbbm{I}_L$, and  $\mathbf{C} \in \mathbb{R}^{L \times L}$ be diagonal. Then, $\text{HSM-SSD}(\mathbf{x}, \mathbf{a}, \mathbf{B}, \mathbf{C})$ is equivalent to $\text{NC-SSD}(\mathbf{x}, \mathbf{a}, \mathbf{B}, \mathbf{C})$ including gating and output projection, as $\mathbf{x}_\text{out} = f(\mathbf{y}) = \mathbf{C} f(\mathbf{h})$. See supplement for proof.
\end{proposition}


\begin{figure}[t]
     \centering
     \includegraphics[width=\columnwidth]{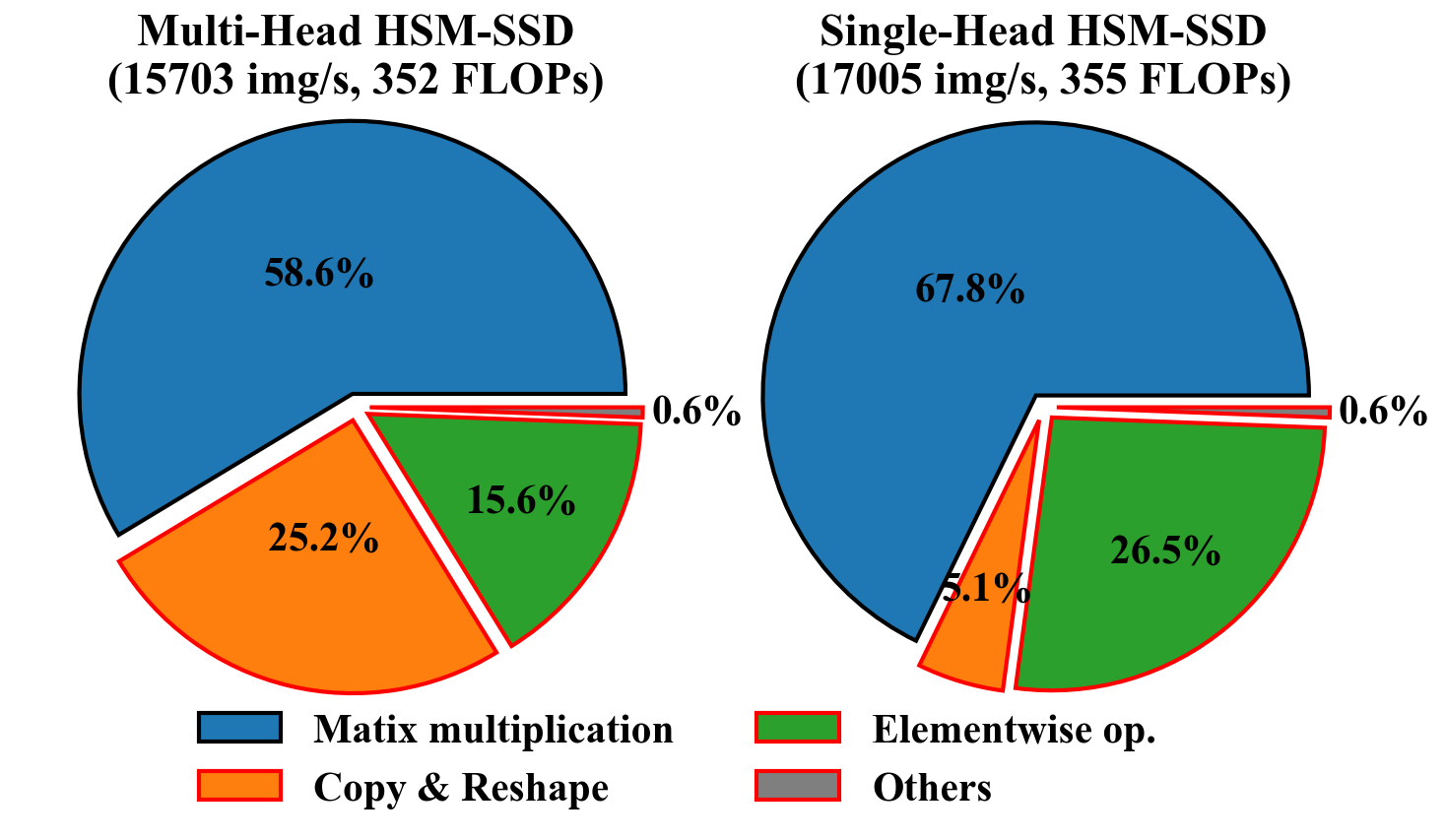}
     \vspace{-15pt}
    \caption{\textbf{Runtime breakdown of HSM-SSD with EfficientViM-M2.} The operations highlighted in red are memory-bound.}
    \label{fig:break}
    \vspace{-10pt}
\end{figure}

\begin{remark}
Although we utilize gating and linear projection for the HSM to emulate the operation in the original SSD layer, other methods are also available. 
For example, inspired by the Perceiver architecture~\cite{jaegle2021perceiver}, we could implement a global token mixer for the function $f$ in Eq.~\eqref{eq:HSM} to efficiently handle high-dimensional inputs using a reduced set of latent variables. 
Moreover, HSM-SSD can recursively apply HSM-SSD itself as the hidden state mixer, further reducing computational complexity and forming a recurrent architecture unrolled across depth.
\end{remark}

\subsection{HSM-SSD layer}
\label{sec:3.2}

\noindent\textbf{Multi-stage hidden state fusion.}
To further improve the performance of EfficientViM, we introduce a \textit{multi-stage hidden-state fusion} (\textbf{MSF}) mechanism that fuses the prediction logits leveraging hidden states from multiple stages of the network.
Let $\{\mathbf{h}^{(s)}\}_{s=1}^S$ denote the hidden states at the last block of each stage $s$, where $S$ is the total number of stages.
For each $\mathbf{h}^{(s)}$, we compute a global representation $\thickhat{\mathbf{h}}^{(s)}$ through a simple average over the hidden states:
\begin{equation} 
\thickhat{\mathbf{h}}^{(s)} = \frac{1}{N} {\sum}_{i=1}^{N} \mathbf{h}^{(s)}_i.
\end{equation}
Then, each global representation $\thickhat{\mathbf{h}}^{(s)} \in \mathbb{R}^{D}$ is normalized and projected to generate its corresponding logits $\mathbf{z}^{(s)} \in \mathbb{R}^{c}$, where $c$ indicates the number of classes.
We set the final logit $\mathbf{z}$ of EfficientViM as a weighted sum of the logits from all stages, including the original logit $\mathbf{z}^{(0)}$ obtained from the output of the last stage, which is defined as\\
\begin{figure*}[t!]
    \centering
     \includegraphics[width=\textwidth]{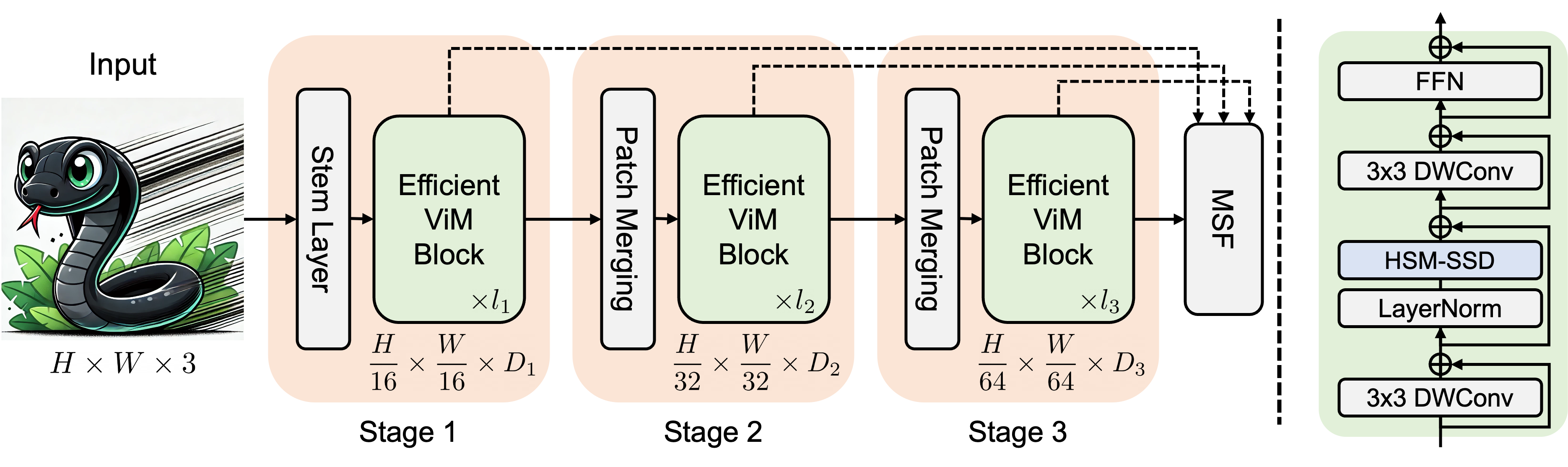}
     \vspace{-15pt}
    \caption{\textbf{(left) Overall architecture and (right) block design of EfficientViM.} The dotted line indicates a skip connection for multi-stage hidden state fusion (MSF). Illustration of the HSM-SSD layer in the EfficientViM block is presented in~\Cref{fig:layer}.}
    \label{fig:arch}
    \vspace{-10pt}
\end{figure*}
\begin{equation} 
\begin{aligned}
\mathbf{z} &= \sum^S_{s=0} \thickhat{\beta}^{(s)} \mathbf{z}^{(s)},\\
\thickhat{\beta}^{(s)} &= \frac{\exp(\beta^{(s)})}{\sum^S_{i=0} \exp(\beta^{(i)})},
\end{aligned}
\end{equation}
where $\beta^{(s)}$ is learnable scalar.
By training with this combined logit, we explicitly reinforce the representational power of the hidden states, as they contribute to the final predictions.
It also enriches the information by integrating both low-level and high-level features, thereby enhancing the generalization ability of the model during inference.
\begin{algorithm}[t]
\caption{HSM-SSD Layer}
\label{algo:HSM-SSD}
\textbf{Input:} $\mathbf{x}_\text{in} \in \mathbb{R}^{L \times D}$\\
\textbf{Output:} $\mathbf{x}_\text{out} \in \mathbb{R}^{L \times D}$
\begin{algorithmic}[1]
\State $\thickhat{\mathbf{B}}, \mathbf{C}, \Delta \leftarrow \text{Linear}(\mathbf{x}_\text{in})$ \Comment{$\mathcal{O}(LND)$}
\State $\thickhat{\mathbf{B}}, \mathbf{C} \leftarrow \text{DWConv}(\thickhat{\mathbf{B}},\mathbf{C})$ \Comment{$\mathcal{O}(LNK^2D)$}
\State $\mathbf{A}, \mathbf{B} \leftarrow \text{Discretization}(\thickhat{\mathbf{a}}, \thickhat{\mathbf{B}}, \Delta)$ \Comment{$\mathcal{O}(LD)$}
\State $\mathbf{h}_\text{in} \leftarrow (\mathbf{A} \odot \mathbf{B})^\top \mathbf{x}_\text{in}$ \Comment{$\mathcal{O}(LND)$}
\State $\mathbf{h}, \mathbf{z} \leftarrow \text{Linear}(\mathbf{h}_\text{in})$ \Comment{$\mathcal{O}(ND^2)$}
\State $\mathbf{h} \leftarrow \text{Linear}(\mathbf{h} \odot \sigma(\mathbf{z}))$ \Comment{$\mathcal{O}(ND^2)$}
\State $\mathbf{x}_\text{out} \leftarrow \mathbf{C}\mathbf{h}$ \Comment{$\mathcal{O}(LND)$}
\State \textbf{Return} $\mathbf{x}_\text{out}$ 
\end{algorithmic}
\end{algorithm}

\vspace{5pt}
\noindent\textbf{Single-head HSM-SSD.}
The multi-head design in the attention mechanism~\cite{vaswani2017attention} allows it to selectively attend to the features from independent representation subspaces within each head.
SSD-based models~\cite{mamba2,shi2024vssd} generally adopt a multi-head variant called multi-input SSD, where the input $\mathbf{x}$ and $\mathbf{a}$ are defined for each head while $\mathbf{B}$, and $\mathbf{C}$ are shared across the head.
However, recent work~\cite{yun2024shvit} has pointed out that a significant portion of real runtime in multi-head self-attention is driven by memory-bound operations. 

\noindent In our preliminary experiments, we also found that multi-head configuration has become a bottleneck for HSM-SSD as summarized in~\Cref{fig:break}.
As shown in the figure, the real runtime of multi-head HSM-SSD is largely bounded by memory access, requiring almost a quarter of the total runtime.
Hence, we eliminate all tensor manipulation caused by multi-head (\eg, reshape, copy operation).
Concurrently, to mimic the capability of multi-head in capturing diverse relationships,
we set $\Delta \in \mathbb{R}^{L \times N}, \thickhat{\mathbf{a}} \in \mathbb{R}^{N}$, enabling the importance weights $\mathbf{A} \in \mathbb{R}^{L \times N}$ to estimate the importance of the tokens per state.
Then, the input for the hidden state mixer becomes
\begin{equation}
     \mathbf{h}_\text{in} = (\mathbf{A} \odot \mathbf{B})^\top  \mathbf{x}_\text{in}.
     \label{eq:final}
\end{equation}
As a result, our single-head HSM-SSD with state-wise importance weights achieved higher throughput (17,005 img/s) with single-head compared to multi-head (15,703 img/s) along with competitive performance (see~\Cref{tab:abl}.b) under similar FLOPs.
The pseudocode of the single-head HSM-SSD layer is provided in~\Cref{algo:HSM-SSD}.

\subsection{EfficientViM}
In this subsection, we present EfficientViM, an efficient vision mambas built upon the HSM-SSD layer. The overall architecture is illustrated in~\Cref{fig:arch}.

\label{sec:3.3}
\noindent\textbf{Block design.}
In each block of EfficientViM, we sequentially stack the HSM-SSD layer and a feed-forward network (FFN) to facilitate global information aggregation and channel interaction, respectively.
FFN layer consists of two consecutive 1$\times 1$ convolution layers, known as pointwise convolution, with an expansion ratio of 4.
To capture the local context with minimal computational costs, we incorporate a 3$\times$3 depthwise convolution (DWConv) layer before both the NC-SSD and FFN layers.
Each layer is combined with a residual connection using a layer scale 
following \cite{touvron2022deit,touvron2021going}.
For normalization, we apply layer normalization (LN) only before the HSM-SSD layer for numerical stability, while batch normalization (BN) is used for DWConv and FFN considering its faster speed over LN.

\noindent\textbf{Overall architecture.}
The stem layer first maps the ${H\times W \times 3}$ input image to the down-sized feature map ${\frac{H}{16} \times \frac{W}{16} \times D_1}$ through the four consecutive 3$\times$3 convolutional layers with the stride of 2.
Then, the resulting feature map is fed into the three stages built with EfficientViM blocks.
To achieve hierarchical architecture and improve efficiency, we downscale the feature map while increasing the number of channels at the end of each stage via the downsampling layer adopted from~\cite{zhu2024vision,liu2023efficientvit,sandler2018mobilenetv2}.
Regarding activation functions, we use SiLU only in the HSM-SSD layer, and others use ReLU~\cite{nair2010rectified} since the latencies of the complex activation functions (\eg, Gelu~\cite{hendrycks2016gaussian}, DynamicReLU~\cite{chen2020dynamic}, etc.) are largely dependent on devices, as discussed in previous works~\cite{vasu2023mobileone,liu2023efficientvit,yun2024shvit}.
Detailed architecture specifications of EfficientViM variants are available in~\Cref{tab:arch}.

\begin{table}[t!]
    \centering
    \footnotesize
    \setlength{\tabcolsep}{5pt}
    \begin{tabular}{c|ccc}
    \toprule
    Model & \# Blocks & \# Channels & \# States  \\
    \midrule
    EfficientViM-M1 & $[2,2,2]$ & $[128,192,320]$ & $[49,25,9]$  \\
    EfficientViM-M2 & $[2,2,2]$ & $[128,256,512]$ & $[49,25,9]$  \\
    EfficientViM-M3 & $[2,2,2]$ & $[224,320,512]$ & $[49,25,9]$\\
    EfficientViM-M4 & $[3,4,2]$ & $[224,320,512]$ & $[64,32,16]$\\
    \bottomrule
    \end{tabular}
    \caption{\textbf{Specification of EfficientViM variants.}}
    \label{tab:arch} 
\vspace{-10pt}
\end{table}

\section{Experiments}
\label{sec:4}
\begin{table*}[t!]
    \centering
    \small
    \footnotesize
    \setlength{\tabcolsep}{5pt}
    \begin{tabular}{l|c|c|c|c|rr|r|c|cc}
    \toprule
    \multicolumn{1}{c|}{\multirow{2}{*}{Method}}& \multicolumn{1}{c|}{\multirow{2}{*}{Venue}}& Input &\multicolumn{1}{c|}{\multirow{2}{*}{Epochs}} &\multicolumn{1}{c|}{Token} &  \multicolumn{2}{c|}{\textbf{Throughput}} & \multicolumn{1}{c|}{\textbf{Latency}} & \multicolumn{1}{c|}{\textbf{Top-1}}  &  \multicolumn{1}{c}{Params}   & \multicolumn{1}{c}{FLOPs}    \\
    \multicolumn{1}{c|}{}&   & Size & \multicolumn{1}{c|}{} & \multicolumn{1}{c|}{Mixer} & \multicolumn{1}{c}{\textbf{(im/s)} $\uparrow$} & \multicolumn{1}{c|}{\textbf{Thr}$_\text{rel}$ $\uparrow$}   & \multicolumn{1}{c|}{\textbf{(ms)} $\downarrow$}  & \multicolumn{1}{c|}{(\%) $\uparrow$}  & \multicolumn{1}{c}{(M)} & \multicolumn{1}{c}{(M)}\\ 
    \midrule
    MobileViTV2 0.5~\cite{mehta2022separable} & Arxiv 2022 & $256^2$ & 300  &  Att.   & 6,702 & $\times$0.32 & 0.149 & 70.2 & 1.4 & 466 \\
    MobileOne-S0~\cite{vasu2023mobileone}  & CVPR 2023& $224^2$ & 300 & Conv & 13,313 &$\times$0.64 & 0.075 & 71.4 &2.1 & 275 \\
    EMO-1M~\cite{zhang2023rethinking}& ICCV 2023 & $224^2$ & 300   &  Att. & 6,945 & $\times$0.34 & 0.144  & 71.5& 1.3   & 261  \\
    MobileFormer-96M~\cite{chen2022mobile} & CVPR 2022 & $224^2$ & 450 &  Att. & 11,554 & $\times$0.56  & 0.087  & 72.8& 4.6 & 96 \\
    SHViT-S1~\cite{yun2024shvit} & CVPR 2024 & $224^2$ & 300 &  Att. & 19,868 & $\times$0.96 & 0.050  & 72.8 &   6.3 & 241 \\
    \rowcolor{lightgray} \textbf{EfficientViM-M1} &  - & $224^2$ & 300  & SSD & \textbf{20,731} & $\times$\textbf{1.00}& \textbf{0.048}  & 72.9 &    6.7 & 239 \\
    MobileNetV3-L 0.75~\cite{howard2019searching} & ICCV 2019 & $224^2$ & 600 & Conv & 10,846 & $\times$0.52 &  0.092 & 73.3 & 4.0 & 155 \\
    EfficientViT-M3~\cite{liu2023efficientvit} & CVPR 2023 & $224^2$ & 300 &  Att. & 16,045 & $\times$0.77& 0.062   & 73.4 & 6.9 & 263\\ 
    \rowcolor{lightgray} \textbf{EfficientViM-M1} & - & $224^2$ & 450  & SSD & \textbf{20,731} & $\times$\textbf{1.00} & \textbf{0.048}  & \textbf{73.5} & 6.7 & 239 \\
    \midrule

    EfficientFormerV2-S0~\cite{li2023rethinking} &NeurIPS 2022 & $224^2$ & 300  &  Att.& 1,350 & $\times$0.08 &0.741  & 73.7 &3.5& 407\\
    EfficientViT-M4~\cite{liu2023efficientvit}  & CVPR 2023 & $224^2$ & 300 &  Att. & 15,807 & $\times$0.93  & 0.063  & 74.3 & 8.8 & 299\\
    EdgeViT-XXS~\cite{pan2022edgevits} & ECCV 2022 & $224^2$ & 300 &  Att.   & 5,990 & $\times$0.35 & 0.167  & 74.4 & 4.1 & 556 \\
    EMO-2M~\cite{zhang2023rethinking} & ICCV 2023 & $224^2$ & 300 &  Att.   & 4,990 & $\times$0.29  & 0.200  & 75.1 & 2.3 & 439 \\
    MobileNetV3-L 1.0~\cite{howard2019searching} & ICCV 2019 & $224^2$ & 600 & Conv & 9,493 & $\times$0.56 &  0.105 & 75.2  & 5.4 & 217  \\
    MobileFormer-151M~\cite{chen2022mobile} & CVPR 2022 & $224^2$ & 450 &  Att.   & 8,890 & $\times$0.52 & 0.112  & 75.2 & 7.6 & 151\\
    SHViT-S2~\cite{yun2024shvit} & CVPR 2024 & $224^2$ & 300  &  Att.   & 15,899 & $\times$0.93 & 0.063  & 75.2& 11.4 & 366  \\
    \rowcolor{lightgray} \textbf{EfficientViM-M2}  &-& $224^2$ & 300 & SSD   & \textbf{17,005} & $\times$\textbf{1.00} & \textbf{0.059}  & 75.4 & 13.9 & 355\\
    MobileViTV2 0.75~\cite{mehta2022separable}& Arxiv 2022 & $256^2$ & 300  &  Att.   &4,409 & $\times$0.26 & 0.227  & 75.6 & 2.9 & 1030\\ 
    FastViT-T8~\cite{vasu2023fastvit}  & ICCV 2023  & $256^2$ & 300  &  Att.   & 4,365 & $\times$0.26 & 0.229  &75.6 &3.6&705 \\
    \rowcolor{lightgray} \textbf{EfficientViM-M2}  &-& $224^2$ & 450 & SSD   & \textbf{17,005} & $\times$\textbf{1.00} & \textbf{0.059}  & \textbf{75.8}  & 13.9 & 355\\
    \midrule
    EfficientMod-XXS~\cite{ma2024efficient} & ICLR 2024 &$224^2$ & 300 &   Att.   & 7022 & $\times$0.59 & 0.142  & 76.0 &  4.7 & 583\\
    ConvNeXtV2-A~\cite{woo2023convnext} & CVPR 2023 & $224^2$ & 300 &  Conv & 7563 & $\times$0.63 & 0.132 & 76.2 & 3.7  &  552 \\
    EfficientViT-M5~\cite{liu2023efficientvit} & CVPR 2023  & $224^2$ & 300 &  Att.   & 11,105 & $\times$0.93  & 0.090 & 77.1 & 12.4 & 522 \\
    MobileOne-S2~\cite{vasu2023mobileone}  & CVPR 2023 & $224^2$ & 300 & Conv & 5,360  & $\times$0.45 & 0.187   & 77.4 & 7.8 &1299 \\
    SHViT-S3~\cite{yun2024shvit}  & CVPR 2024 & $224^2$ & 300  &  Att.   & 11,873 & $\times$0.99 & 0.084 & 77.4 &   14.2 & 601 \\
    EdgeViT-XS~\cite{pan2022edgevits} & ECCV 2022 & $224^2$ & 300  &  Att.   & 4,405 & $\times$0.37 & 0.227 & 77.5  & 6.7 & 1136 \\
    \rowcolor{lightgray} \textbf{EfficientViM-M3}&- & $224^2$& 300  & SSD   & \textbf{11,952} & $\times$\textbf{1.00} & \textbf{0.084}  & 77.6 & 16.6 & 656\\
    MobileFormer-294M~\cite{chen2022mobile} & CVPR 2022 & $224^2$ & 450 &  Att.   & 6,576 & $\times$0.55 & 0.152  & 77.9 &11.4 & 294\\
    EfficientFormerV2-S1~\cite{li2023rethinking} & NeurIPS 2022 & $224^2$  & 300 &  Att.  & 1,248 & $\times$0.10 & 0.801  & 77.9&  6.1 & 668 \\
    \rowcolor{lightgray} \textbf{EfficientViM-M3} &- & $224^2$ & 450 & SSD   & \textbf{11,952} & $\times$\textbf{1.00} & \textbf{0.084} & \textbf{77.9} & 16.6 & 656  \\
    \midrule
    ConvNeXtV2-F~\cite{woo2023convnext} & CVPR 2023 & $224^2$ & 300 &  Conv & 6,405 & $\times$0.78 & 0.156 & 78.0 &  5.2 & 785  \\
    MobileViTV2 1.0~\cite{mehta2022separable}& Arxiv 2022 & $256^2$  & 300 &  Att.    & 2,977 & $\times$0.36 & 0.336 & 78.1 &  4.9 & 1844 \\ 
    MobileOne-S3~\cite{vasu2023mobileone}  & CVPR 2023 & $224^2$ & 300& Conv & 4,181 & $\times$0.51 & 0.239 & 78.1 &10.1 &1896  \\
    EfficientMod-XS~\cite{ma2024efficient}  & ICLR 2024 &$224^2$ & 300&   Att.   & 5,321 & $\times$0.65 &  0.188 & 78.3 & 6.6 & 778 \\
    EMO-6M~\cite{zhang2023rethinking} & ICCV 2023 & $224^2$ & 300 &  Att.   & 3,266 & $\times$0.40 & 0.306 & 79.0 & 6.1 & 961  \\
    FastViT-T12~\cite{vasu2023fastvit}  & ICCV 2023  & $256^2$ & 300&  Att.   & 2,741 & $\times$0.34  & 0.365 &79.1  &6.8 &1419 \\
    MobileFormer-508M~\cite{chen2022mobile} & CVPR 2022 & $224^2$ & 450 &  Att.   & 4,586 & $\times$0.56 & 0.218 & 79.3  & 14.0 & 508  \\
    MobileOne-S4~\cite{vasu2023mobileone}  & CVPR 2023 & $224^2$ & 300 & Conv & 3,041 & $\times$0.37 & 0.329 & 79.4 &  14.8 &2978  \\
    SHViT-S4~\cite{yun2024shvit}  & CVPR 2024 & $256^2$ & 300 &  Att.   & 8,024 & $\times$0.98 & 0.124 & 79.4 &16.5 & 986  \\
    \rowcolor{lightgray} \textbf{EfficientViM-M4}&-  & $256^2$ & 300 & SSD   & \textbf{8,170} & $\times$\textbf{1.00} & \textbf{0.122}  & 79.4 & 19.6 & 1111 \\
    MobileViTV2 1.25~\cite{mehta2022separable}& Arxiv 2022 & $256^2$  & 300 &  Att.   & 2,409 & $\times$0.24 & 0.415 & 79.6 & 7.5 & 2857 \\ 
    \rowcolor{lightgray}\textbf{EfficientViM-M4}&- & $256^2$ & 450 & SSD   & \textbf{8,170} & $\times$\textbf{1.00} &\textbf{0.122}  & \textbf{79.6}  & 19.6 & 1111 \\

    \bottomrule
    \end{tabular}
    \caption{\textbf{Comparison of efficient networks on ImageNet-1K~\cite{deng2009imagenet} classification.} Results are sorted by accuracy. We also denote the relative throughput \textbf{Thr}$_\text{rel}$ of each method compared to EfficientViM in each split.}
    \label{tab:main} 
\vspace{-15pt}
\end{table*}

In this section, we first demonstrate the effectiveness of EfficientViM on image classification (\Cref{sec:4.1}).
Then, we conduct experiments to analyze the extensibility of EfficientViM on dense predictions (\Cref{sec:4.2}). 
We also provide ablation studies and analysis of EfficientViM (\Cref{sec:4.3}). 
See the supplement for implementation details and more experiments.
\subsection{Image Classification.}
\label{sec:4.1}
\noindent\textbf{Comparison with efficient vision backbones.}
For comparison of EfficientViM with prior works, we conduct ImageNet-1K~\cite{deng2009imagenet} classification.
To validate the effectiveness of EfficientViM in speed-accuracy trade-offs, we measure the throughput (im/s), and latency (ms) with the batch size of 256 on an NVIDIA RTX 3090 along with the accuracy, and present the results in~\Cref{tab:main}.
EfficientViM outperforms all previous efficient networks in both speed and accuracy.
After training 450 epochs, EfficientViM-M1 shows a competitive performance with MobileNetV3-L 0.75~\cite{howard2019searching} and EfficientViT-M3~\cite{liu2023efficientvit} while achieving about 90\% and 30\% speedup.
Further, EfficientViM-M2 achieves about 4$\times$ faster speed, with a 0.2\% performance improvement compared to MobileViTV2 0.75~\cite{mehta2022separable} and FastViT-T8~\cite{vasu2023fastvit}.
EfficientViM-M3 and M4 achieve 77.9\% and 79.7\% accuracy, respectively, surpassing all previous works in throughput and accuracy within each section.
Also, EfficientViM consistently outperforms the previous SOTA network, SHViT~\cite{yun2024shvit}, across model sizes while reducing latency, which demonstrates the superiority of EfficientViM.

\begin{table}[t!]
\vspace{15pt}
    \centering
    \footnotesize
    \setlength{\tabcolsep}{3pt}
    \begin{tabular}{l|rcc|cc}
    \toprule
    \multicolumn{1}{c|}{Method}& \multicolumn{1}{c}{\textbf{Thr.}} & \textbf{Thr}$_\text{rel}$ &  \textbf{Top-1} & Params & FLOPs \\
    \midrule
    \rowcolor{lightgray}\textbf{EfficientViM-M2} & 17,005& \textbf{$\times$2.08} & 75.8 & 13.9M & 355M  \\
    ViM-T~\cite{zhu2024vision}  & 1,612  & $\times$0.20 & 76.1  & 7.1M & 1500M  \\
    LocalViM-T~\cite{zhu2024vision} & 593 & $\times$0.07 & 76.2  & 8.0M &  1500M \\
    EfficientVMamba-T~\cite{ma2024efficient} & 2,763& $\times$0.34& 76.5  & 6.0M  &  800M  \\
    MSVMamba-N~\cite{shi2024multi}  & 2,060& $\times$0.25& 77.3  &  6.9M & 864M  \\
    EfficientVMamba-S~\cite{ma2024efficient} & 1,350& $\times$0.17 & 78.7& 11.0M  & 1300M  \\
    \rowcolor{lightgray}\textbf{EfficientViM-M4} & 8,170 & $\times $\textbf{1.00}& 79.6 & 19.6M & 1111M \\
    MSVMamba-M~\cite{shi2024multi} & 1,527 & $\times$0.19 & 79.8 & 11.9M &  1507M   \\
    \bottomrule
    \end{tabular}
    \caption{\textbf{Comparison of EfficientViM with vision Mambas.} \textbf{Thr}$_\text{rel}$ is relative throughput compared to EfficientViM-M4.}
    \label{tab:ssm} 
    \vspace{-10pt}
\end{table}

\noindent\textbf{Comparison with vision Mambas.}
In~\Cref{tab:ssm}, we compare our EfficientViM with the recent vision Mambas including ViM~\cite{zhu2024vision}, LocalViM~\cite{huang2024localmamba}, EfficientVMamba~\cite{pei2024efficientvmamba}, and MSVMamba~\cite{shi2024multi}.
EfficientViM brings promising speed improvements over the prior works.
EfficientViM-M2 is almost 10$\times$ and 29$\times$ faster than ViM-T and LocalViM-T with comparable accuracy.
EfficientViM-M4 shows about 3$\sim$14$\times$ higher throughput than other methods even with competitive performances.
Notably, it outperforms MSVMamba-N and EfficientVMamba-S by 3.1\% and 2.3\%, achieving about 4$\times$ and 6$\times$ higher throughput, respectively.
This reveals that EfficientViM is a highly efficient architecture among mamba-based vision models.

\noindent\textbf{Training with distillation.}
We compare EfficientViM with the prior works~\cite{li2023rethinking, shaker2023swiftformer,vasu2023fastvit,yun2024shvit} trained with distillation objectives in DeiT~\cite{touvron2021training}.
We train the model for 300 epochs, using RegNetY-160~\cite{radosavovic2020designing} as the teacher model.
\Cref{tab:dist} shows that distillation is effective for EfficientViM.
Compared to FastViT-T8\&T12~\cite{vasu2023fastvit}, EfficientViM-M2\&M4 delivers more than 3$\times$ higher throughput along with comparable or even better performance.
Further, EfficientViM outperforms SHViT~\cite{yun2024shvit} up to 0.8\% running at a higher speed.
With distillation, EfficientViM still outperforms all other models in speed-accuracy trade-offs and further establishes a promising Pareto front as shown in~\Cref{fig:intro}.
\begin{table}[t!]
    \centering
    \footnotesize
    \setlength{\tabcolsep}{1.5pt}
    \begin{tabular}{l|c|rcc|cc}
    \toprule
    \multicolumn{1}{c|}{Method} & Size & \multicolumn{1}{c}{\textbf{Thr.}} & \textbf{Thr}$_\text{rel}$  & \textbf{Top-1} & Params & FLOPs \\
    \midrule
    SHViT-S1~\cite{yun2024shvit} &$224^2$& 19,868  & $\times$0.96 & 74.0 & 6.3M  & 241M  \\
    \rowcolor{lightgray}\textbf{EfficientViM-M1} & $224^2$ & \textbf{20,731} & $\times$\textbf{1.00} & \textbf{74.6} & 6.7M  & 239M   \\
    \midrule
    EfficientFormerV2-S0~\cite{li2023rethinking} &$224^2$& 1,350& $\times$0.08 & 75.7 & 3.5M & 407M  \\
    SwifitFormer-XS~\cite{shaker2023swiftformer} &$224^2$& 6,102 & $\times$0.36  &75.7 & 3.5M &  605M\\
    SHViT-S2~\cite{yun2024shvit} &$224^2$& 15,672& $\times$0.94 & 76.2 & 11.4M & 366M \\
    FastViT-T8~\cite{vasu2023fastvit}&$256^2$& 4,365& $\times$0.26 & 76.7  & 3.6M & 705M \\
    \rowcolor{lightgray}\textbf{EfficientViM-M2} & $224^2$ & \textbf{17,005} & $\times$\textbf{1.00} & \textbf{76.7} & 13.9M & 355M  \\
    \midrule
    EfficientMod-XXS~\cite{ma2024efficient} & $224^2$ & 7,022& $\times$0.59  & 77.1 & 4.7M  & 583M \\
    SHViT-S3~\cite{yun2024shvit} &$224^2$& 11,873& $\times$0.99   & 78.3& 14.2M & 601M \\
    SwifitFormer-S~\cite{shaker2023swiftformer} &$224^2$& 4,675& $\times$0.39  &78.5 & 6.1M & 988M \\
    EfficientFormerV2-S1~\cite{li2023rethinking} &$224^2$& 1,248 & $\times$0.10  &79.0  & 6.1M& 668M  \\
    \rowcolor{lightgray}\textbf{EfficientViM-M3} & $224^2$ & \textbf{11,952} & $\times$\textbf{1.00}  & \textbf{79.1} & 16.6M & 656M  \\
    \midrule
    EfficientMod-XS~\cite{ma2024efficient}  & $224^2$ & 5,321 & $\times$0.65 & 79.4 & 6.6M  & 778M  \\
    SHViT-S4~\cite{yun2024shvit} &$256^2$& 8,024 & $\times$0.98  & 80.2 & 16.5M & 986M \\
    FastViT-T12~\cite{vasu2023fastvit}&$256^2$& 2,741 & $\times$0.34  & 80.3 & 6.8M & 1419M \\
    \rowcolor{lightgray}\textbf{EfficientViM-M4} & $256^2$ & \textbf{8,170} & $\times$\textbf{1.00} & \textbf{80.7} & 19.6M & 1111M  \\
    \bottomrule
    \end{tabular}
    \vspace{-5pt}
    \caption{\textbf{Comparison of efficient networks after training with distillation objective in~\cite{touvron2021training}.} \textbf{Thr}$_\text{rel}$ is relative throughput compared to EfficientViM in each split.}
    \label{tab:dist} 
    \vspace{-10pt}
\end{table}

\vspace{5pt}

\begin{table}[t!]
    \centering
    \resizebox{\columnwidth}{!}{%
    \footnotesize
    \setlength{\tabcolsep}{2pt}
    \begin{tabular}{l|c|cccccc}
    \toprule
    \multicolumn{8}{c}{\textit{Head: Mask R-CNN~\cite{he2017mask}}}\\
    \midrule
        \multicolumn{1}{c|}{Method} &\textbf{Lat. (ms)} & \textbf{AP}$^\text{b}$ & \textbf{AP}$^\text{b}_{50}$ & \textbf{AP}$^\text{b}_{75}$ & \textbf{AP}$^\text{m}$ & \textbf{AP}$^\text{m}_{50}$ & \textbf{AP}$^\text{m}_{75}$\\
    \midrule
    
    EfficientNet-B0~\cite{tan2019efficientnet} & 0.95 & 31.9 & 51.0 & 34.5 & 29.4 & 47.9 & 31.2\\
    PoolFormer-S1~\cite{yu2022metaformer} & 1.49 &37.3 & 59.0  &40.1 & 34.6 & 55.8 & 36.9\\
    FastViT-SA12~\cite{vasu2023fastvit} & 1.63 & 38.9 & 60.5 & 42.2 & \textbf{35.9} & 57.6 & \textbf{38.1} \\
    SHViT-S4~\cite{yun2024shvit} &  0.52 & 39.0 & \textbf{61.2} & 41.9 & \textbf{35.9} & \textbf{57.9} & 37.9\\
    \rowcolor{lightgray}\textbf{EfficientViM-M4} & \textbf{0.45}& \textbf{39.3}& 60.2 & \textbf{42.5} & 35.8 & 57.1 &  37.4\\
    \midrule
    \multicolumn{8}{c}{\textit{Head: RetinaNet~\cite{ross2017focal}}}\\
    \midrule
    \multicolumn{1}{c|}{Method} & \textbf{Lat. (ms)} & \textbf{AP} & \textbf{AP}$_{50}$ & \textbf{AP}$_{75}$ & \textbf{AP}$_\text{s}$ & \textbf{AP}$_\text{m}$ & \textbf{AP}$_\text{l}$\\
    \midrule
    PVTV2-B0~\cite{wang2021pvtv2} & 0.87 & 37.2 & 57.2 & 39.5 & \textbf{23.1} & 40.4 & 49.7 \\
    MobileFormer-508M~\cite{chen2022mobile} & 1.09 & 38.0 & 58.3 & 40.3 & 22.9 & 41.2 & 49.7 \\
    EdgeViT-XXS~\cite{pan2022edgevits} &  0.94 & 38.7 & 59.0 & 41.0 & 22.4 & 42.0 & 51.6\\
    SHViT-S4~\cite{yun2024shvit} & 0.52 & 38.8 &\textbf{ 59.8} & \textbf{41.1} & 22.0 & \textbf{42.4 }& 52.7\\
    \rowcolor{lightgray}\textbf{EfficientViM-M4} & \textbf{0.45} & \textbf{38.8} & 59.6 & \textbf{41.1 }& 22.1 & \textbf{42.4 }& \textbf{52.8}\\
    \bottomrule
    \end{tabular}}
    \vspace{-5pt}
    \caption{\textbf{Instance segmentation and object detection and results on COCO-2017~\cite{lin2014microsoft}.}}
    \label{tab:dense} 
    \vspace{-10pt}
\end{table}

\subsection{Dense Predictions}
\label{sec:4.2}
\noindent\textbf{Object detection and instance segmentation.}
We validate the effectiveness of EfficientViM on object detection and instance segmentation using the COCO-2017~\cite{lin2014microsoft} dataset.
For training the models, we follow the settings of previous works~\cite{vasu2023fastvit,yun2024shvit,liu2023efficientvit, pan2022edgevits}, where Mask R-CNN~\cite{he2017mask} and RetinaNet~\cite{ross2017focal} are used for instance segmentation and object detection.
After training the model for 12 epochs (1x schedule) with a batch size of 16, we report the performances and backbone latencies with the resolution of $512^2$, following~\cite{yun2024shvit}.
\Cref{tab:dense} demonstrates that EfficientViM achieves competitive performances while maintaining a faster speed in dense prediction tasks.
Specifically, in instance segmentation with Mask R-CNN, EfficientViM-M4 surpasses SHViT-S4 with the 0.3\% improvements in $\textbf{AP}^\text{b}$ while reducing latency by 0.7ms.
Similarly, in object detection with RetinaNet, EfficientViM-M4 achieves the best average precision of 38.8\% with the lowest latency.
\begin{table}[t!]
\centering
\setlength{\tabcolsep}{5pt}
\footnotesize
\begin{tabular}{l|c|c|c|c}
\toprule
\multicolumn{1}{c|}{\multirow{2}{*}{Metric}}  & PVTV2 & FastViT & EdgeViT & \cellcolor{lightgray}\textbf{EfficientViM} \\
& -B0 & -SA12 & -XXS &\cellcolor{lightgray} \textbf{-M4}\\
\midrule
Latency (ms) & 0.87 & 1.63 & 0.94 & \cellcolor{lightgray}\textbf{0.45} \\
mIoU(\%) & 37.2 & 38.0 & 39.7 & \cellcolor{lightgray}\textbf{41.3} \\
\bottomrule
\end{tabular}
\vspace{-5pt}
\caption{\textbf{Semantic segmentation with SemanticFPN~\cite{kirillov2019panoptic} on ADE20K~\cite{zhou2017scene}}}
\label{tab:seg}
\vspace{-10pt}
\end{table}

\noindent\textbf{Semantic segmentation.}
We also demonstrate the extensibility of EfficientViM on semantic segmentation using ADE20K~\cite{zhou2017scene} benchmark.
Following previous works~\cite{vasu2023fastvit,pan2022edgevits}, we replace the backbone of SemanticFPN~\cite{kirillov2019panoptic} with EfficientViM-M4.
The model is finetuned for 40K iterations with the batch size of 32 and the initial learning rate of 2$\times$10$^{-4}$ decayed using a polynomial scheduler with a power of 0.9.
We report the mIoU and latency of the backbone with 512$^2$ resolutions in~\Cref{tab:seg}.
We observe that EfficientViM-M4 largely surpasses all previous efficient backbones in semantic segmentation, which demonstrates the efficacy of EfficientViM in dense predictions along with the results of object detection and instance segmentation.
It is worth noting that EfficientViM-M4 brings both 1.6\% mIoU gains and about 2$\times$ speedup compared to the second-best model EdgeVit-XXS.

\begin{table}[t!]
\centering
\setlength{\tabcolsep}{2pt}
\footnotesize
\begin{tabular}{l|rrcc|c}
\toprule
\multicolumn{1}{c|}{Method} &  \multicolumn{1}{c}{\textbf{Memory}} & \multicolumn{1}{c}{\textbf{Thr.}} & \textbf{Thr}$_\text{rel}$  & \textbf{Top-1} & Params \\
\midrule
EfficientViT-M4~\cite{liu2023efficientvit} & 870M & 15,807 & $\times$0.93 & 74.3 & 8.8M \\
EMO-2M~\cite{zhang2023rethinking} & 2656M & 4,990  & $\times$0.29 & 75.1 & 2.3M \\
MobileNetV3-L 1.0~\cite{howard2019searching} &  2643M & 9,493 & $\times$0.56 & 75.2 & 5.4M \\
Mobile-Former-151M~\cite{chen2022mobile} &  1800M & 8,890 & $\times$0.52 & 75.2 & 7.6M  \\
SHViT-S2~\cite{yun2024shvit} & 879M & 15,899  & $\times$0.94 & 75.2 & 11.4M  \\
FastViT-T8~\cite{yun2024shvit} & 2811M & 4,365 & $\times$0.26 & 75.6  & 3.6M \\
\rowcolor{lightgray}\textbf{EfficientViM-M2} & 969M & 17,005 & $\times$\textbf{1.00} & 75.8 & 13.9M \\
\bottomrule
\end{tabular}
\vspace{-5pt}
\caption{\textbf{Comparsion on peak memory usage during inference}. \textbf{Thr}$_\text{rel}$ is relative throughput compared to EfficientViM-M2.}
\label{tab:memory}
\vspace{-15pt}
\end{table}


\subsection{Analysis and Ablation studies }
\label{sec:4.3}
\noindent\textbf{Memory Efficiency.}
Our models have a relatively large number of parameters compared to prior works.
Yet, the memory usage of the model in the device is determined by the memory I/O during inference rather than just the number of parameters alone.
We here analyze the peak memory usage of the networks and provide the results in~\Cref{tab:memory}.
Despite the highest parameter counts, EfficientViM shows a competitive memory efficiency while achieving the best throughput.
Notably, we observe that EfficientViM requires only about 1/3 of the peak memory usage of the models having lower parameters, \eg, EMO~\cite{zhang2023rethinking}, MobileNetV3~\cite{howard2019searching}, and FastViT~\cite{vasu2023fastvit}.
Furthermore, the models with a relatively large number of parameters (\eg, EfficientViT, SHViT, and EfficientViM) demonstrate high throughput and low memory consumption, highlighting that the number of parameters is not a critical factor for memory and time efficiency.
Overall, our EfficientViM achieves the best speed and performance maintaining memory efficiency.\\
\noindent\textbf{Ablation studies.}
We conduct ablation studies with EfficientViM-M2 after training 300 epochs and provide the results in~\Cref{tab:abl}.
First, we compare HSM-SSD with other global token mixers, by replacing them with other methods including NC-SSD~\cite{shi2024vssd} and self-attention (SA)~\cite{vaswani2017attention}.
Considering that EfficientViM-M3 with HSM-SSD shows 77.5\% with a throughput of 11,952 (im/s), NC-SSD and SA show a poorer speed-accuracy trade-off than HSM-SSD.
Regarding head choices, we observe that our single-head design brings significant speed-up (15703 $\rightarrow$ 17005 (im/s)) without performance degradation.
Lastly, multi-stage fusion with hidden states (75.4\%) surpasses the accuracy of EfficientViM without fusion (75.1\%) under similar throughput.
Refer to the supplement for more ablation studies.
\begin{table}[t!]
\centering
\setlength{\tabcolsep}{3pt}
\footnotesize
\begin{tabular}{cl|rc|cc}
\toprule
\multicolumn{2}{c|}{Method} & \multicolumn{1}{c}{\textbf{Thr.}} & \textbf{Top-1} & Params & FLOPs \\
\midrule
\rowcolor{lightgray} \red{(A)} & \textbf{EfficientViM-M2} (Base) & 17,005 & 75.4 & 13.9M & 355M  \\
\rowcolor{lightgray} \rowcolor{lightgray} \red{(B)} & \textbf{EfficientViM-M3} & 11,952 & 77.5 & 16.6M & 656M  \\
\midrule
\multicolumn{6}{l}{\textit{a. Token Mixers (Base: HSM-SSD)}}\\
\midrule
\color[HTML]{1f77b4}{(C)} & $(\rightarrow)$ {NC-SSD}~\cite{shi2024vssd} & 9,786 & 76.2 & 13.0M & 382M  \\
\color[HTML]{ff7f0e}{(D)} &$(\rightarrow)$ Self-Attention~\cite{vaswani2017attention} & 13,038 & 76.1 & 13.6M & 416M  \\
\midrule
\multicolumn{6}{l}{\textit{b. Head designs (Base: Single-head with $\mathbf{A} \in \mathbb{R}^{L \times N}$)}}\\
\midrule
\color[HTML]{2ca02c}{(E)} & $(\rightarrow)$ Multi-head & 15,703 & 75.4 & 13.9M & 352M  \\
\midrule
\multicolumn{6}{l}{\textit{c. Multi-stage fusion}}\\
\midrule
\color[HTML]{9467bd}{(F)} &$(\rightarrow)$ None & 17,317  & 75.1& 13.0M & 354M \\
\midrule
\end{tabular}
\vspace{-5pt}
\caption{\textbf{Ablation studies on EfficientViM.} All ablation studies were conducted with EfficientViM-M2 denoted as (Base).}
\label{tab:abl}
\vspace{-10pt}
\end{table}






\section{Related Works}
\noindent\textbf{Efficient vision backbones.}
Earlier works~\cite{howard2017mobilenets,sandler2018mobilenetv2,howard2019searching,iandola2016squeezenet,tan2019efficientnet,ma2018shufflenet,zhang2018shufflenet,chollet2017xception, vasu2023mobileone} have studied improving the trade-off between accuracy and computational efficiency in CNN architectures.
To reduce the complexity of convolution, Xception~\cite{chollet2017xception} introduced depthwise convolutions (DWConv), which have become major techniques used across modern efficient models.
For instance, the pioneering work MobileNet~\cite{howard2017mobilenets} constructed lightweight architectures based on DWConv.
ShuffleNet~\cite{zhang2018shufflenet} and GhostNet~\cite{han2020ghostnet} also have explored additional techniques to enhance CNN by shuffling the channel and generating more feature maps.\\
Following ViTs~\cite{dosovitskiy2020image}, several works~\cite{mehta2021mobilevit, mehta2022separable,vasu2023fastvit, zhang2023rethinking, chen2022mobile, liu2023efficientvit,yun2024shvit,pan2022edgevits} built efficient vision backbones with attention~\cite{vaswani2017attention}.
Some works sparsified the attention~\cite{liu2021swin,chu2021twins,dong2022cswin} to reduce query and key, while another line of works~\cite{wang2020linformer,kitaev2020reformer,xiong2021nystromformer,choromanski2020rethinking} have approximated attention itself with reduced cost.
More recently, EfficientFormer~\cite{li2022efficientformer}, EfficientViT~\cite{liu2023efficientvit}, and SHViT~\cite{yun2024shvit} have proposed hardware-frienly ViT architectures for real-world applications, focusing on actual speed in practical use, rather than FLOPs.\\
\noindent\textbf{Vision Mambas.}
State space model (SSM)~\cite{fu2022hungry,gu2021efficiently,smith2022simplified,mamba,mamba2} has become a popular global token mixer with its favorable linear cost.
Especially, Mamba~\cite{mamba} has introduced a selective scan mechanism on SSM (S6) to enable time-variant selection.
Following Mambas~\cite{mamba,mamba2}, several works~\cite{zhu2024vision,liu2024vmamba,tang2024scalable,shaker2024groupmamba,yang2024plainmamba,huang2024localmamba,ma2024efficient} have proposed SSM-based Vision backbones.
ViM~\cite{zhu2024vision} applied bi-directional SSM to flattened patches considering the causal nature in Mamba.
Similarly, VMamba~\cite{liu2024vmamba}, PlainMamba~\cite{yang2024plainmamba}, and LocalMamba~\cite{huang2024localmamba} have been proposed with the multi-path scanning mechanism.
GroupMamba~\cite{shaker2024groupmamba}, and EfficientVMamba~\cite{pei2024efficientvmamba} further enhance the efficiency by splitting the channels and sharing the learnable parameters for each path.
Recent works, VSSD~\cite{shi2024vssd} and Linfusion~\cite{liu2024linfusion} introduced non-causal SSD to overcome causal properties in Mamba.

\section{Conclusion}
We propose a novel mamba-based architecture, Efficient Vision Mamba (EfficientViM), built with a hidden state mixer-based SSD (HSM-SSD).
With the observation that the primary bottleneck of SSD stems from the linear projections with input tokens, we rearrange the channel mixing operation within the hidden states which serve as reduced latent representations.
We also introduce multi-stage hidden state fusion that integrates both low-level and high-level features.
To enhance practical runtime efficiency, we adopt a single-head design for HSM-SSD to minimize the memory-bound operations, which are typically overlooked when only considering FLOPs.
Our comprehensive experiments demonstrate that EfficientViM significantly improves efficiency and performance over prior works across various tasks.

\paragraph{Acknowledgments.}
This work was supported by Korea University - KT (Korea Telecom) R\&D Center, the National Research Foundation of Korea (NRF) grant funded by the Korea government (MSIT) (NRF-2023R1A2C2005373), the Virtual Engineering Platform Project (Grant No. P0022336), funded by the Ministry of Trade, Industry \& Energy (MoTIE, South Korea), and the Institute of Information \& communications Technology Planning \& Evaluation (IITP) grant funded by the Korean government (MSIT) (No. RS-2024-00457882, AI Research Hub Project).



\maketitlesupplementary
\setcounter{section}{0}
\renewcommand{\thesection}{\Alph{section}}
\setcounter{table}{0}
\renewcommand{\thetable}{\Alph{table}}
\setcounter{figure}{0}
\renewcommand{\thefigure}{\Alph{figure}}

\begin{table}[t!]
    \centering
    \small
    \setlength{\tabcolsep}{5pt}
    \begin{tabular}{c|ccc}
    \toprule
    Configuration & Base & Distillation & FT  \\
    \midrule
    Epochs & 300/450 & 300 & 30 + 30  \\
    Batch size & \multicolumn{2}{c}{2048} & 1024  \\
    Weight decay & \multicolumn{2}{c}{0.05} & 1e-8 \\
    Warmup Epochs & \multicolumn{2}{c}{20} & 0 \\
    Cooldown Epochs & \multicolumn{2}{c}{10} & 0  \\
    Learning rate & \multicolumn{2}{c}{2e-3} & 1e-3 \\
    Min Learning rate & \multicolumn{2}{c}{2e-5} & 1e-5  \\
    Optimizer (Momentum) & \multicolumn{3}{c}{Adamw (0.9, 0.999)}\\
    Gradient Clipping & \multicolumn{3}{c}{0.02}\\
    Learning rate scheduler & \multicolumn{3}{c}{Cosine} \\
    Rand Augment & \multicolumn{3}{c}{rand-m9-mstd0.5-inc1} \\
    Mixup & \multicolumn{3}{c}{0.8} \\
    Cutmix & \multicolumn{3}{c}{1.0} \\
    Mixup switch prob & \multicolumn{3}{c}{0.5} \\
    Random erasing prob & \multicolumn{3}{c}{0.25} \\
    Label smoothing & \multicolumn{3}{c}{0.1} \\
    EMA decay rate & \multicolumn{3}{c}{0.9995} \\
    Teacher model & None & RegNetY-160 & None \\
    
    \bottomrule
    \end{tabular}
    \caption{\textbf{Settings for training EfficientViM.} FT: finetuning with higher resolution images (\Cref{sec:B}).}
    \label{tab:impl} 
\end{table}

\section{Implementation Details}
We use the ImageNet-1K~\cite{deng2009imagenet} to validate the effectiveness of EfficientViM on the image classification task.
For training EfficientViM, we follow the training recipes of previous works~\cite{touvron2022deit,shi2024vssd,yun2024shvit}.
Specifically, all models are trained from scratch with a batch size of 2,048 for 300 epochs using AdamW optimizer~\cite{loshchilov2017decoupled} with a warmup of 20 epochs and a cooldown of 10 epochs.
Following~\cite{li2023rethinking,chen2022mobile, howard2019searching}, we also report the results after training 450 epochs.
During training, we adopt a cosine annealing~\cite{loshchilov2016sgdr} scheme with the initial learning rate of $2 \times 10^{-3}$ decreasing to $2 \times 10^{-5}$.
The weight decay of 0.05 and gradient clipping with a threshold of 0.02 are used.
Also, MESA~\cite{du2022sharpness} and EMA with the decay rate of 0.9995 is adopted following~\cite{shi2024vssd,han2024demystify}.
For data augmentation, we follow DeiT~\cite{touvron2021training} using Mixup~\cite{zhang2017mixup} \& CutMix~\cite{yun2019cutmix} with a Label smoothing~\cite{szegedy2016rethinking}, RandAugment~\cite{cubuk2019autoaugment}, and Random Erasing~\cite{zhong2020random}.
We report the throughput and latency with the batch size of 256 on Nvidia RTX 3090 GPU.

\noindent Additionally, we finetune the model with the batch size of 1024, using cosine annealing with the initial learning rate of $1 \times 10^{-3}$, for 30 epochs at a resolution of $384^2$, followed by an additional 30 epochs at $512^2$.
For a fair comparison, we employ pre-trained models trained for 300 epochs.
Also, to report the throughput of the models with extremely high-resolution images in~\Cref{fig:scale}, we start with a batch size of 256 and halve it once the memory exceeds the GPU limit as the resolution increases.
Regarding training with distillation, all settings are the same as in Table 3. of the main paper except for the guidance from the teacher model of RegNetY-160~\cite{radosavovic2020designing} following DeiT~\cite{touvron2021training}.


\begin{table}[t!]
    \centering
    \footnotesize
    \setlength{\tabcolsep}{3pt}
    \begin{tabular}{l|c|ccc|cc}
    \toprule
    \multicolumn{1}{c|}{Method} & Size & \textbf{Thr.} & \textbf{Thr}$_\text{rel}$ & \textbf{Top-1} & Params & FLOPs\\
    \midrule
    SHViT-S4~\cite{yun2024shvit} & $384^2$ & 3,685 & $\times$0.99 & 81.0 & 16.5M & 2225M \\
    \rowcolor{lightgray}\textbf{EfficientViM-M4}&  $384^2$ & 3,724 & $\times$1.00 & 80.9 & 21.3M & 2379M \\
    \midrule
    SHViT-S4~\cite{yun2024shvit} &$512^2$ & 2,122 & $\times$0.86  &82.0 & 16.5M & 3973M\\
    \rowcolor{lightgray}\textbf{EfficientViM-M4}& $512^2$ & 2,452 & $\times$1.00 & 81.9 & 21.3M & 4154M\\
    \bottomrule
    \end{tabular}
    \caption{\textbf{Classification results on ImageNet-1K~\cite{deng2009imagenet} after finetuning with higher resolutions.} \textbf{Thr}$_\text{rel}$ is relative throughput compared to EfficientViM-M4.}
    \label{tab:high} 
    \vspace{-5pt}
\end{table}

\begin{figure}[t]
     \centering
     \includegraphics[width=0.9\columnwidth]{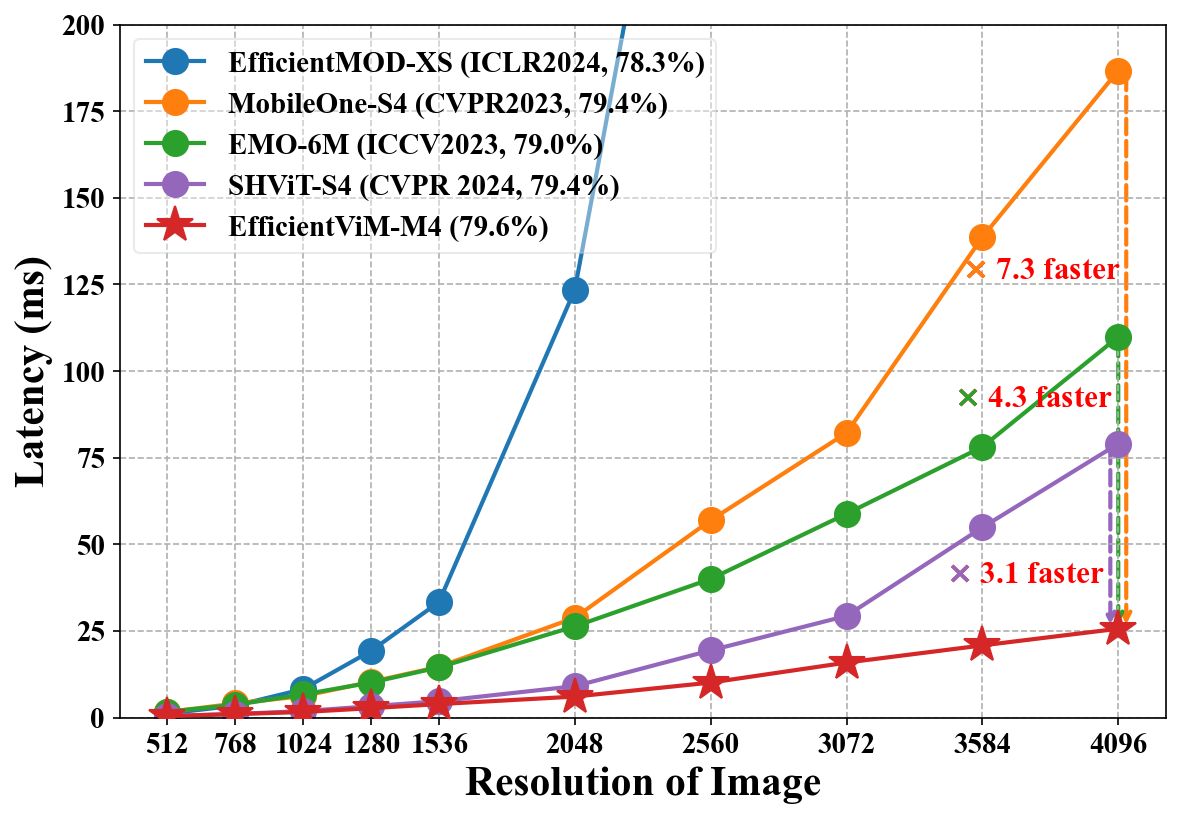}
    \vspace{-10pt}
    \caption{\textbf{Latency comparison of recent efficient networks for an extremely high-resolution image.}}
    \label{fig:scale}
    \vspace{-15pt}
\end{figure}
\section{EfficientViM with high-resolution images.}
\label{sec:B}
Following~\cite{yun2024shvit}, we also explore the applicability of EfficientViM on high-resolution images after fine-tuning 30 epochs at a resolution of $384^2$, followed by an additional 30 epochs at $512^2$.
For a fair comparison, we use the EfficientViM-M4 pre-trained for 300 epochs.
In $384^2$ size, EfficientViM demonstrates competitive performance as presented in~\Cref{tab:high}.
Interestingly, when the resolution increases, the throughput gap between EfficientViM and SHViT~\cite{yun2024shvit} gets larger, resulting in more than 15\% speedup compared to SHViT in $512^2$ while achieving comparable accuracy.
We further investigate the scalability of our method in extremely high-resolution images beyond $512^2$, by comparing the latency of the models while scaling the resolution from $512^2$ to $4096^2$.
As depicted in~\Cref{fig:scale}, we observe an advantage of EfficientViM over the recent state-of-the-art method on extremely high-resolution images.
EfficientViM shows about 3$\times$, 4$\times$, and 7$\times$ faster speed compared to SHViT, EMO~\cite{zhang2023rethinking}, and MobileOne~\cite{vasu2023mobileone}, respectively.
This notable result highlights the scalability of EfficientViM for high-resolution images based on linear cost of HSM-SSD. 


\begin{table}[t!]
\centering
\setlength{\tabcolsep}{3pt}
\footnotesize
\begin{tabular}{cl|rc|cc}
\toprule
\multicolumn{2}{c|}{Method} & \multicolumn{1}{c}{\textbf{Thr.}} & \textbf{Top-1} & Params & FLOPs \\
\midrule
\rowcolor{lightgray}\red{(A)} &\textbf{EfficientViM-M2} (Base) & 17,005 & 75.4 & 13.9M & 355M  \\
\rowcolor{lightgray}\red{(B)} & \textbf{EfficientViM-M3} & 11,952 & 77.5 & 16.6M & 656M  \\
\midrule
\multicolumn{5}{l}{\textit{a. Token Mixers (Base: HSM-SSD)}}\\
\midrule
\color[HTML]{1f77b4}{(C)} &$(\rightarrow)$ NC-SSD~\cite{shi2024vssd} & 9,786 & 76.2 & 13.0M & 382M  \\
\color[HTML]{ff7f0e}{(D)} &$(\rightarrow)$ Self-Attention~\cite{vaswani2017attention} & 13,038 & 76.1 & 13.6M & 416M  \\
\midrule
\multicolumn{5}{l}{\textit{b. Head designs (Base: Single-head (SH) with $\mathbf{A} \in \mathbb{R}^{L \times N}$)}}\\
\midrule
\color[HTML]{2ca02c}{(E)} &$(\rightarrow)$ Multi-head & 15,703 & 75.5 & 13.9M & 352M  \\
\color[HTML]{9467bd}{(F)} &$(\rightarrow)$ SH w. $\mathbf{a}\in \mathbb{R}^{L}$ & 17,081   & 75.2& 13.9M & 352M \\
\midrule
\multicolumn{5}{l}{\textit{c. Multi-stage fusion (Base: $\mathbf{h}^{(s)}$)}}\\
\midrule
\color[HTML]{e377c2}{(G)} &$(\rightarrow)$ Fusion with $\mathbf{x}^{(s)}$& 17,041  & 75.3 & 13.4M & 355M \\
\color[HTML]{7f7f7f}{(H)} &$(\rightarrow)$ None & 17,317  & 75.1& 13.0M & 354M \\
\midrule
\multicolumn{5}{l}{\textit{d. \# states ($N$) of each stage (Base: $[49,25,9]$)}}\\
\midrule
\color[HTML]{bcbd22}{(I)} &$(\rightarrow) [9,25,49]$ & 16,476  & 75.4 & 14.0M & 407M \\
\color[HTML]{17becf}{(J)} &$(\rightarrow) [25,25,25]$ & 16,991 & 75.2 & 13.9M  & 373M   \\
\midrule
\multicolumn{5}{l}{\textit{e. Normalization (Base: Partial LN)}}\\
\midrule
(L) &$(\rightarrow)$ Full BN & 17,432 &\textit{NaN} & 13.9M & 355M \\
\bottomrule
\end{tabular}
\caption{\textbf{Ablation studies on EfficientViM.} All ablations are conducted with EfficientViM-M2. See \Cref{fig:abl}
for a visualization comparing the ablated models with the Pareto front of EfficientViM.}
\label{tab:abl2}
\vspace{-10pt}
\end{table}

\begin{figure}[t!]
    \centering
    \begin{subfigure}{0.8\columnwidth}
     \centering
     \includegraphics[width=\columnwidth]{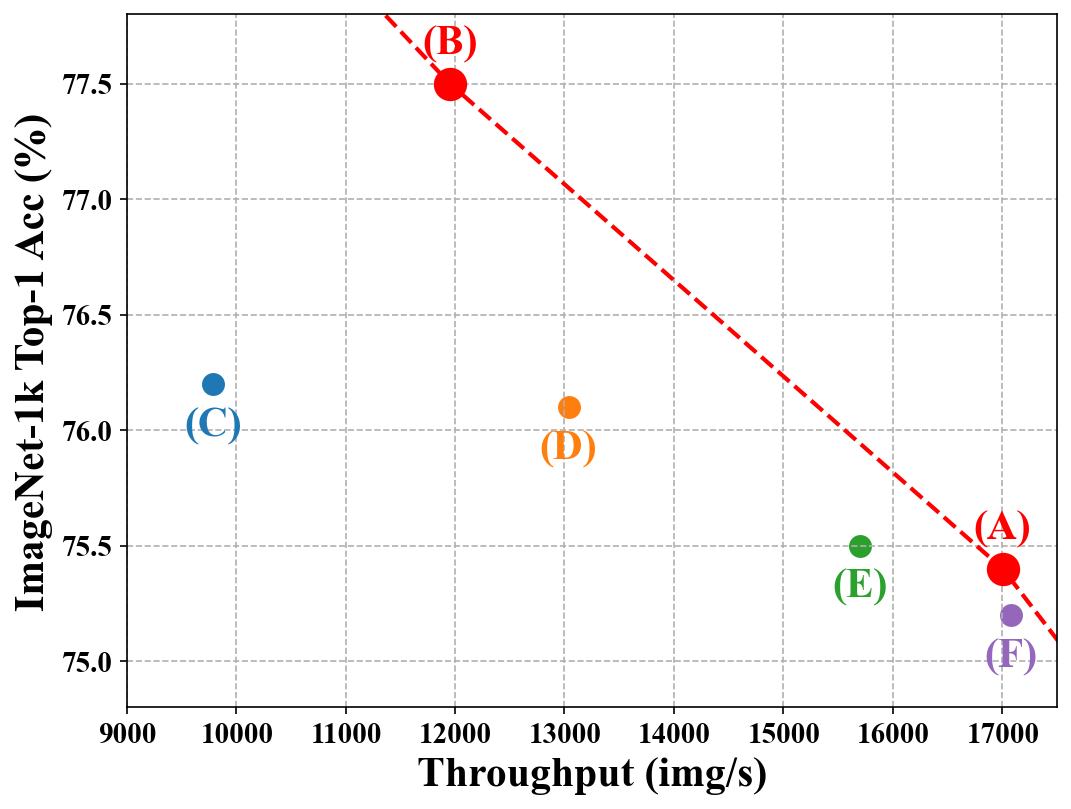}
    \caption{\textbf{Ablations on the token mixer and head design (C-F).}}
    \label{fig:abl1}
    \end{subfigure}
    \vspace{10pt}
    \begin{subfigure}{0.8\columnwidth}
     \centering
     \includegraphics[width=\columnwidth]{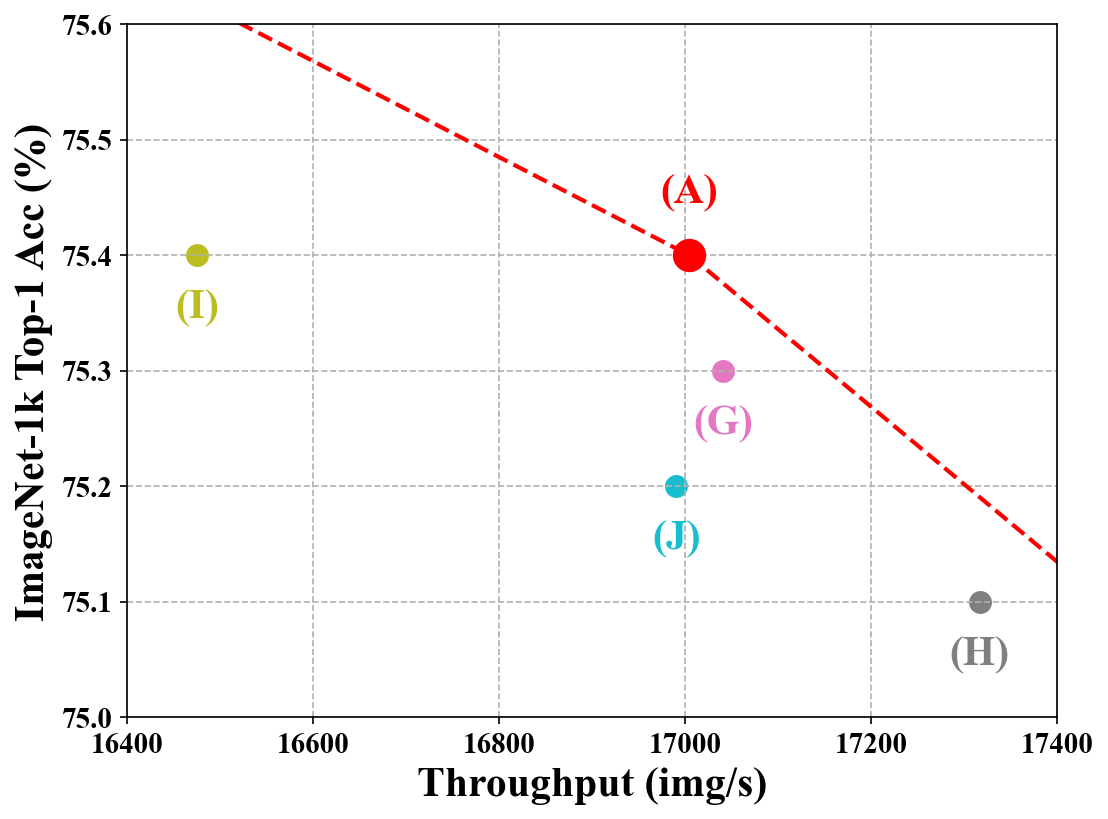}
    \caption{\textbf{Ablations on Multi-stage fusion and \# states (G-J).}}
    \label{fig:abl2}
    \end{subfigure}

    \caption{\textbf{Ablation studies on EfficientViM.} Refer to~\Cref{tab:abl2} for the corresponding models. \red{Red} line indicates the Pareto frontier of EfficientViM.}
    \label{fig:abl}
\end{figure}

\section{Ablation Studies}
Here, we show the effectiveness of HSM-SSD by ablating the proposed components.
The results are summarized in~\Cref{tab:abl2}, and~\Cref{fig:abl}.
First, we compare HSM-SSD with other global token mixers, by replacing them with other methods including (C) NC-SSD~\cite{shi2024vssd} and (D) self-attention (SA)~\cite{vaswani2017attention}.
Considering that EfficientViM-M3 with HSM-SSD shows 77.5\% with a throughput of 11,952 (im/s), NC-SSD and SA show a poorer speed-accuracy trade-off than HSM-SSD.
Regarding head choices, we observe that our single-head design brings significant speed-up (17005 (im/s)) over (E) multi-head (15,703 (im/s)) without performance degradation.
Additionally, defining the importance score per state as $\mathbf{A}\in\mathbb{R}^{L \times N}$ to mimic multi-head leads to the +0.2\% gain with a minor increase in latency, compared to using the original score (F) $\mathbf{a} \in \mathbb{R}^L$.
Note that all ablates models (C-F) are placed under the Pareto front of EfficientViM (\Cref{fig:abl1}), which proves the efficacy of HSM-SSD and our single-head design.\\
Also, multi-stage fusion with hidden states (75.4\%) surpasses the accuracy of (G) the fusion with the output feature maps $\mathbf{x}^{(s)}$ (75.3\%) and 
(H) the EfficientViM without fusion (75.1\%) under similar throughput.
For the number of states $N$, we observe that an increasing schedule with respect to the stages is more effective than (I) a decreasing or (J) constant schedule. 
See (G)-(J) in~\Cref{tab:abl2} and~\Cref{fig:abl2} for the ablation studies on multi-stage fusion and the number of states.
Additionally, for normalization, using (L) batch normalization (BN) across all operations is simple and fast, but, this approach leads to numerical instability.
Therefore, we apply layer normalization (LN) only before HSM-SSD, and BN for the rest.

\begin{table}[t]
    \centering
    \footnotesize
    \setlength{\tabcolsep}{3pt}
    \begin{tabular}{l|c|rc|cc}
    \toprule
    \multicolumn{1}{c|}{Method} & Token Mixer & \multicolumn{1}{c}{\textbf{Thr.}} &   \textbf{Top-1} & Params & FLOPs \\
    \midrule
    VSSD-M~\cite{shi2024vssd} & NC-SSD & 1459 & 82.5 & 14M & 2.3G \\
    VSSD-T~\cite{shi2024vssd} & NC-SSD & 947 & 84.1 & 24M & 4.5G \\
    \rowcolor{lightgray} VSSD-T & $\rightarrow$ \textbf{HSM-SSD} & 1660 &82.7 & 24M& 3.7G\\
    \bottomrule
    \end{tabular}
    \caption{\textbf{Comparison of HSM-SSD with NC-SSD.}}

    \label{tab:ssm2} 
\end{table}

\begin{figure*}[t!]
 \centering
    \begin{subfigure}{0.16\linewidth}
        \centering
        \includegraphics[width=\linewidth]{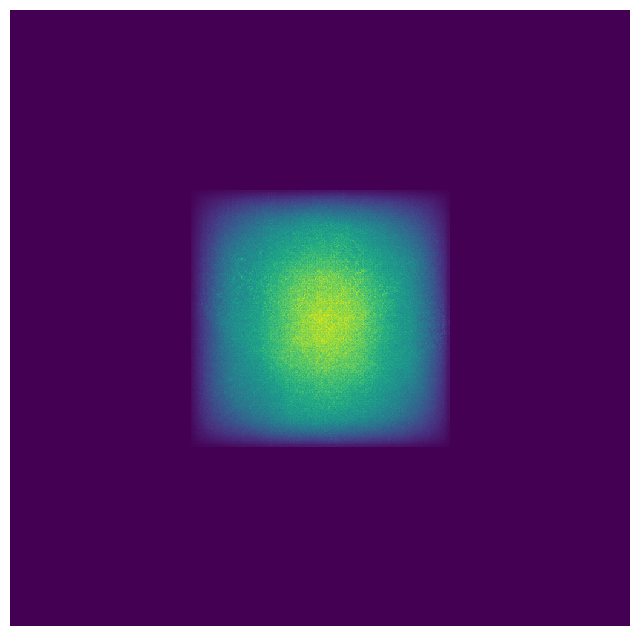}
        \caption{ResNet50}
    \end{subfigure}
    \hfill
    \begin{subfigure}{0.16\linewidth}
        \centering
        \includegraphics[width=\linewidth]{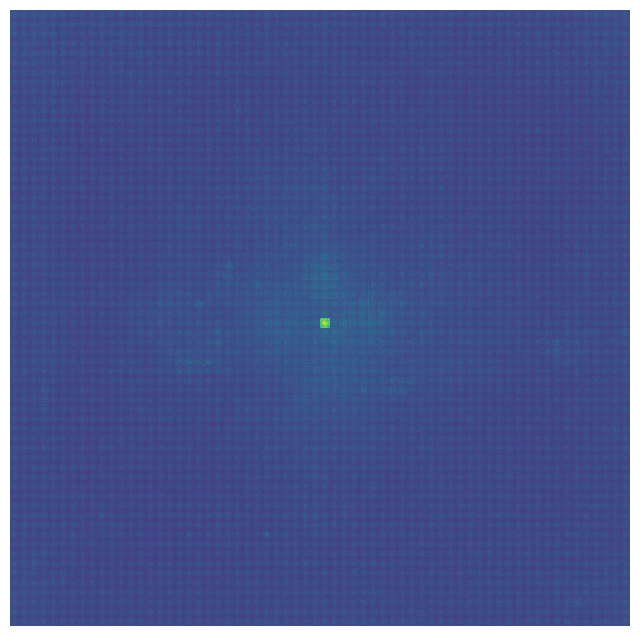}
        \caption{DeiT-S}
    \end{subfigure}
    \hfill
    \begin{subfigure}{0.16\linewidth}
        \centering
        \includegraphics[width=\linewidth]{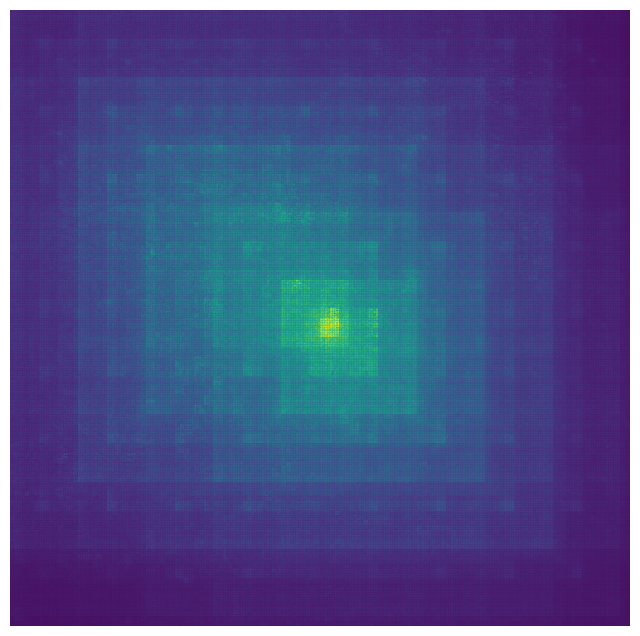}
        \caption{Swin-T}
    \end{subfigure}
    \hfill
    \begin{subfigure}{0.16\linewidth}
        \centering
        \includegraphics[width=\linewidth]{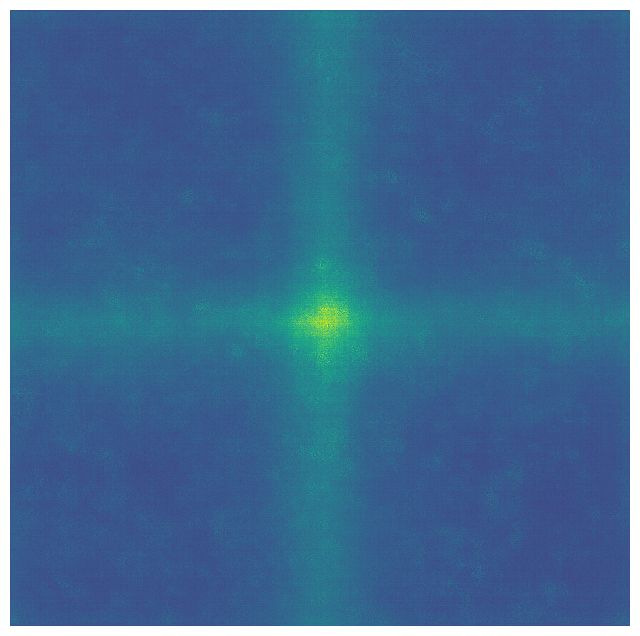}
        \caption{VMamba-T}
    \end{subfigure}
    \hfill
    \begin{subfigure}{0.16\linewidth}
        \centering
        \includegraphics[width=\linewidth]{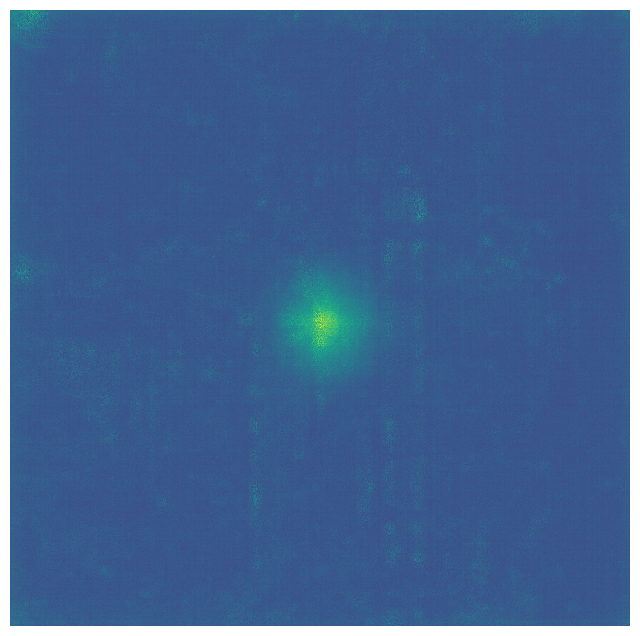}
        \caption{VSSD-T}
    \end{subfigure}
    \hfill
    \begin{subfigure}{0.16\linewidth}
        \centering
        \includegraphics[width=\linewidth]{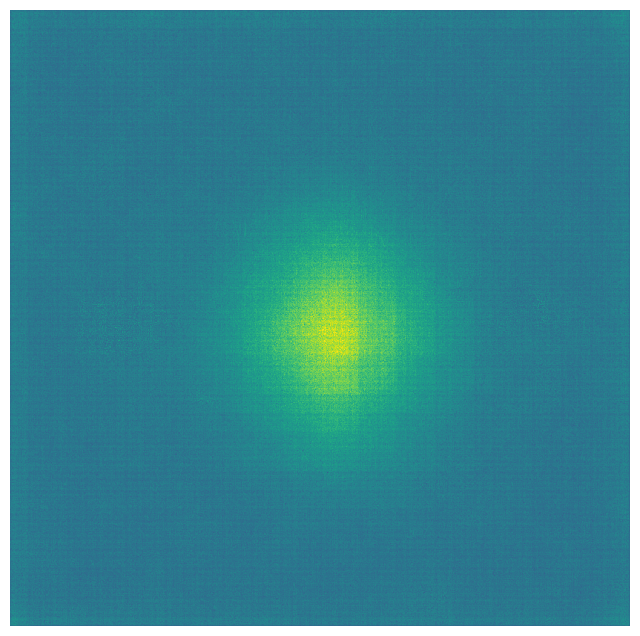}
        \caption{EfficientViM-M4}
    \end{subfigure}
    \caption{\textbf{Comparison of Effective Receptive Fields (ERF)~\cite{luo2016understanding}}}
    \label{fig:erf}
\end{figure*}
\section{Comparison of HSM-SSD with NC-SSD}
To show the advantage of the proposed HSM-SSD over NC-SSD, we replace NC-SSD of VSSD-T~\cite{shi2024vssd} with HSM-SSD and train the model following the original training configuration provided in the official repository.
In Table~\ref{tab:ssm2}, after replacing the token mixer with HSM-SSD, VSSD-T with HSM-SSD demonstrates a significant increase in the throughput (1.8$\times$).
Notably, compared to scaling down the models to smaller sizes (\eg, VSSD-M), replacing the token mixer with HSM-SSD provides better speed-accuracy trade-offs (14\% faster, +0.2\% accuracy), highlighting the advantage of HSM-SSD over NC-SSD.
We also provide the qualitative comparison of HSM-SSD with NC-SSD in the following section.

\section{Effective Receptive Field of HSM-SSD}
In this section, we qualitatively compare the HSM-SSD with previous token mixers including convolutions in ResNet50~\cite{he2016deep}, attentions in vision Transformers such as DeiT-S~\cite{touvron2021training} and Swin-T~\cite{liu2021swin}, and SSM and SSD variants in vision Mambas like VMamba-T~\cite{liu2024vmamba} and VSSD-T~\cite{shi2024vssd}.
We here analyze the Effective Receptive Fields (ERF)~\cite{luo2016understanding} of each model, which quantifies the region of the input that contributes to the output.
In~\Cref{fig:erf}, we visualize the ERF with respect to the central pixel of the output feature maps.
Among various models, EfficientViM-M4 with HSM-SSD shows a global receptive field rather than focusing on a specific region.
For instance, ResNet50 shows a relatively small ERF due to its intrinsic locality of convolution.
The attention mechanism in DeiT-S predominantly focuses on the central pixel itself, and the shifted window attention in Swin-T limits the global receptive field.
Further, since SSM is conducted after flattening the image patch both vertically and horizontally in VMamba-T, it generates an unnatural cross pattern in ERF.
VSSD-T shows a relatively better global receptive field, yet it still largely depends on the close region.
On the other hand, EfficientViM-M4 generates a global effective receptive field (ERF) similar to that of VSSD-T but extracts more information from all regions, enabling it to capture the global dependencies better.

\section{CPU \& Mobile Latency}
To understand the applicability of EfficientViM in a resource-constrained environment, we here provide the latencies of vision backbones in GPU, CPU, and mobile devices (\Cref{tab:edge}).
The latencies are measured with a batch size of 256 on an NVIDIA RTX 3090 GPU, 16 on an AMD EPYC 7742 CPU, and 1 on an iPhone 16 (iOS 18.1).
For mobile latencies, we use CoreML~\cite{Core_ML} library.
EfficientViM-M4 achieves the highest accuracy of 79.6\% with the lowest latency on GPUs and competitive latency on CPU and iPhone 16.
Although EfficientViM-M4 shows slightly higher latency than a few of the prior works in CPU and mobile, EfficientViM-M4 shows a significantly lower latency as the resolution increases (\Cref{fig:scale2}).
In the resolution of $2048^2$, EfficientViM-M4 achieves at least 58\% and 20\% reductions in latency compared to previous works in iPhone 16 and CPU, respectively.
To summarize, EfficientViM serves as a general solution suitable for both GPU and edge devices. 
Furthermore, EfficientViM is an effective backbone for real-world applications where high-resolution images are given, such as in image generation, object detection, and instance segmentation.
\begin{table}[t!]
\centering
\setlength{\tabcolsep}{5pt}
\footnotesize
\begin{tabular}{l|crc|l}
\toprule
\multicolumn{1}{c|}{\multirow{2}{*}{Method}} & \multicolumn{3}{c|}{Latency (ms)}  & \multicolumn{1}{c}{\multirow{2}{*}{\textbf{Top-1}}} \\
& GPU & \multicolumn{1}{c}{CPU} & Mobile &\\
\midrule
MobileViTV2 1.0~\cite{mehta2022separable} & 0.34 & 138.8 & 1.1 & 78.1 \\
EfficientMod-XS~\cite{ma2024efficient} & 0.19 & 33.1 & 0.9 & 78.3 (79.4)\\
MobileFormer-508M~\cite{chen2022mobile} & 0.22 & 29.7 & 2.1 & 79.3 \\
FastViT-T12~\cite{vasu2023fastvit} & 0.37 & 81.3 & 1.8 & 79.1 (80.3) \\
MobileOne-S4~\cite{vasu2023mobileone} & 0.33 & 79.9 & 1.0 & 79.4 \\
SHViT-S4~\cite{yun2024shvit} & 0.12 & 27.4 & 0.9 & 79.4 (80.2) \\
\rowcolor{lightgray} \textbf{EfficientViM-M4} & 0.12 & 32.1 & 1.0 & 79.6 (80.7) \\
\bottomrule
\end{tabular}
\caption{\textbf{Latency comparison of EfficientViM-M4 with prior works.} The number in parentheses indicates the performance with distillation.}
\label{tab:edge}
\end{table}

\begin{figure*}[t!]
    \centering
    \begin{subfigure}{0.48\linewidth}
     \centering
     \includegraphics[width=\columnwidth]{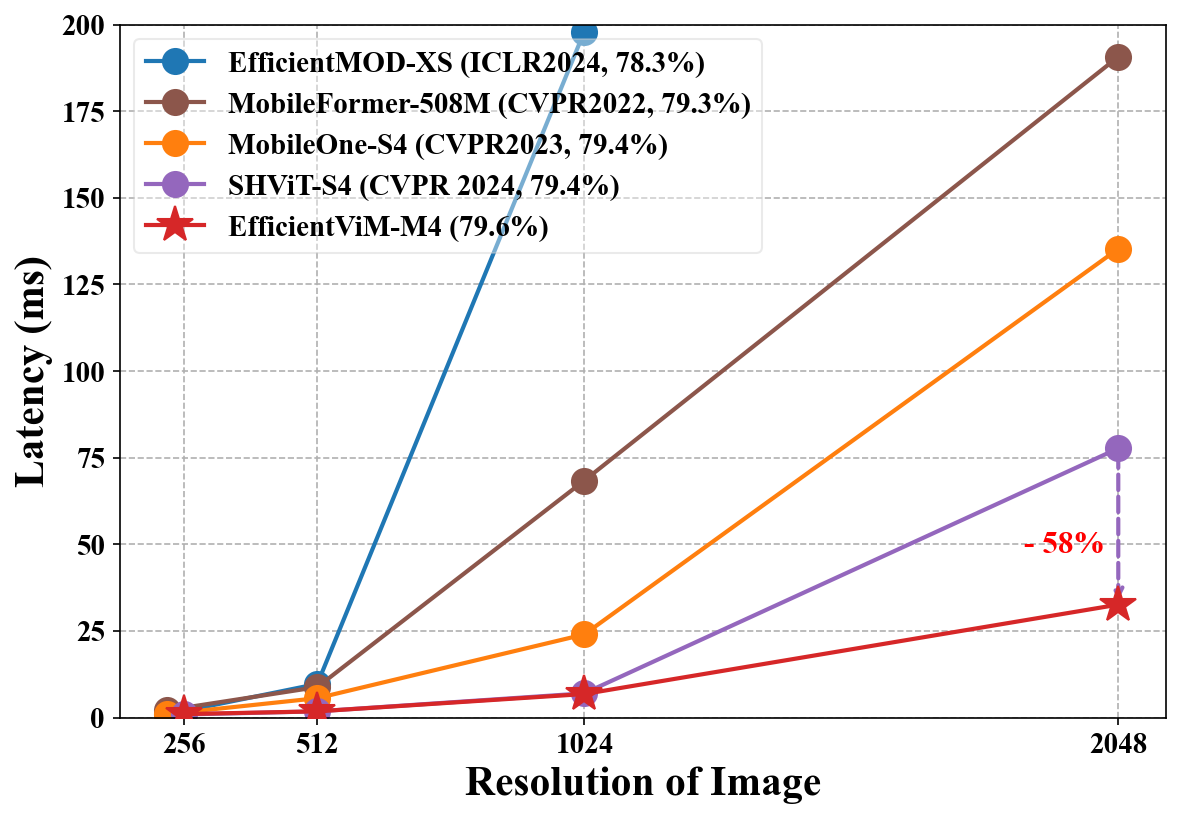}
    \caption{\textbf{Mobile latency comparison.}}
    \label{fig:scale_mobile}
    \end{subfigure}
    \hfill
    \begin{subfigure}{0.48\linewidth}
     \centering
     \includegraphics[width=\columnwidth]{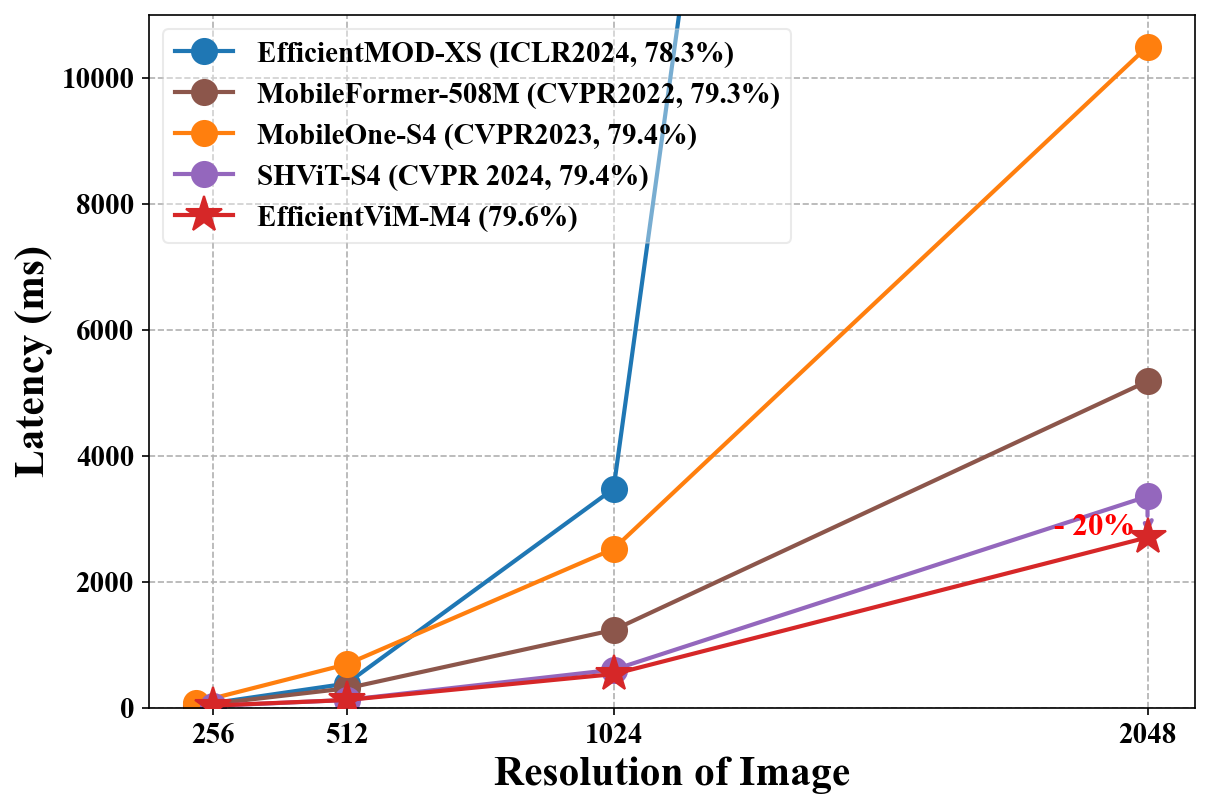}
    \caption{\textbf{CPU latency comparison.}}
    \label{fig:scale_cpu}
    \end{subfigure}

    \caption{\textbf{Mobile and CPU latency comparison.}}
    \label{fig:scale2}
\end{figure*}

\setcounter{equation}{9}

\section{Proof for Proposition 1}
\setcounter{proposition}{0} 
\begin{proposition}
Let $N=L$, $\mathbf{a} \mathbbm{1}^\top_L \odot \mathbf{B} = \mathbbm{I}_L$, and  $\mathbf{C} \in \mathbb{R}^{L \times L}$ be diagonal. Then, $\text{HSM-SSD}(\mathbf{x}, \mathbf{a}, \mathbf{B}, \mathbf{C})$ is equivalent to $\text{NC-SSD}(\mathbf{x}, \mathbf{a}, \mathbf{B}, \mathbf{C})$ including gating and output projection, as $\mathbf{x}_\text{out} = f(\mathbf{y}) = \mathbf{C} f(\mathbf{h})$.
\end{proposition}

\begin{proof}

\noindent It is sufficient to show that $\mathbf{C}f(\mathbf{h})$ of HSM-SSD is equivalent to $f(\mathbf{y})=  (\mathbf{C}\mathbf{h} \odot \sigma(\mathbf{x}_\text{in}\mathbf{W}_\mathbf{z})) \mathbf{W}_\text{out}$ in order to prove the proposition.
Here, based on the assumption, the following holds:
\begin{align*}
\mathbf{C}&f\left(\mathbf{h}\right) \\
&=\mathbf{C}\left((\mathbf{h} \odot \sigma(\mathbf{h}_\text{in} \mathbf{W}_\mathbf{z}))\mathbf{W}_\text{out}\right)\\
&=\mathbf{C}\left((\mathbf{h} \odot \sigma(((\mathbf{a} \mathbbm{1}^\top_L \odot \mathbf{B})^\top \mathbf{x}_\text{in}) \mathbf{W}_\mathbf{z}))\mathbf{W}_\text{out}\right)\\
&=\mathbf{C}(\mathbf{h} \odot \sigma( \mathbf{x}_\text{in} \mathbf{W}_\mathbf{z}))\mathbf{W}_\text{out}\\
&=(\mathbf{C}\mathbf{h} \odot \sigma( \mathbf{x}_\text{in} \mathbf{W}_\mathbf{z}))\mathbf{W}_\text{out} = f(\mathbf{y})
\end{align*}
\end{proof}
\noindent\textit{Remarks.} We assume that $\mathbf{C}$ is a diagonal matrix but $\mathbf{C}$ is dependent on $\mathbf{x}_\text{in}$ since $\mathbf{C} = \mathbf{x}_\text{in}\mathbf{W}_\mathbf{C}$. Unfortunately, there does not exist a weight matrix $\mathbf{W}_\mathbf{C}$ that makes $\mathbf{C}$ diagnoal for arbitrary inputs $\mathbf{x_{in}}$.
We provide this proposition to understand the  relationship between the HSM-SSD and NC-SSD operations. 
This implies that in specific conditions with particular data, the two methods yield the same result. However, one approach does not generalize the other.


\section{Discussion of multi-stage hidden stage fusion}
In this section, we briefly discuss how multi-stage hidden state fusion (MSF) provides a performance boost, although the earlier layers generally provide less accurate logits.
Note that our MSF leverages the logits across the layer as \emph{deep supervision} and \emph{multi-scale representation} to improve the performance.
MSF aligns with the concept of `deep supervision' in pioneering works, such as DSN~\cite{lee2015deeply}, and U-Net++~\cite{zhou2018unet++}.
Specifically, during training, MSF can be interpreted as auxiliary classification tasks, encouraging even the earlier layers to learn more discriminative features.
Further, the earlier layers generally capture fine-grained patterns, while later layers extract high-level semantics, i.e., DenseNet~\cite{huang2017densely} and FPN~\cite{lin2017feature}.
By combining these complementary representations with the learnable coefficients, we take advantage of the ensemble effect from `multi-scale representations'.
In fact, it is well-known that an ensemble of the models often outperforms each model, highlighting the benefits of HSM.
As a results, MSF brings substantial improvements on EfficientViM.

{
    \small
    \bibliographystyle{ieeenat_fullname}
    \bibliography{main}
}


\end{document}